\documentclass[nohyperref]{article}

\usepackage{amsmath,amsfonts,bm}

\def\eqref#1{equation~\ref{#1}}

\def\1{\bm{1}}

\def\va{{\bm{a}}}

\def\vf{{\bm{f}}}
\def\vg{{\bm{g}}}
\def\vh{{\bm{h}}}

\def\vm{{\bm{m}}}

\def\vp{{\bm{p}}}
\def\vq{{\bm{q}}}

\def\vs{{\bm{s}}}

\def\vu{{\bm{u}}}

\def\vx{{\bm{x}}}
\def\vy{{\bm{y}}}

\def\vsigma{{\bm{\sigma}}}

\def\mD{{\bm{D}}}

\def\mF{{\bm{F}}}

\def\mM{{\bm{M}}}

\def\mS{{\bm{S}}}

\def\mU{{\bm{U}}}

\def\mW{{\bm{W}}}

\DeclareMathAlphabet{\mathsfit}{\encodingdefault}{\sfdefault}{m}{sl}
\SetMathAlphabet{\mathsfit}{bold}{\encodingdefault}{\sfdefault}{bx}{n}

\newcommand{\R}{\mathbb{R}}

\newcommand{\hangin}{\goodbreak\hangindent=.35cm \noindent}

\usepackage{microtype}
\usepackage{graphicx}
\usepackage{subfigure}
\usepackage{booktabs} 

\usepackage{hyperref}

\usepackage[accepted]{icml2023}

\usepackage{amsmath}
\usepackage{amssymb}
\usepackage{mathtools}
\usepackage{amsthm}

\usepackage[capitalize,noabbrev]{cleveref}

\theoremstyle{plain}
\newtheorem{theorem}{Theorem}[section]

\theoremstyle{definition}

\theoremstyle{remark}

\usepackage[textsize=tiny]{todonotes}

\usepackage{xfrac}
\usepackage{color,soul}

\icmltitlerunning{Utility-based Perturbed Gradient Descent}

\begin{document}

\twocolumn[
\icmltitle{\large{Utility-based Perturbed Gradient Descent: An Optimizer for Continual Learning}}



\icmlsetsymbol{equal}{*}

\begin{icmlauthorlist}
\icmlauthor{Mohamed Elsayed}{school,amii}
\icmlauthor{A. Rupam Mahmood}{school,amii,cifar}
\end{icmlauthorlist}

\icmlaffiliation{school}{Department of Computing Science, University of Alberta, Edmonton, Canada}

\icmlaffiliation{amii}{Alberta Machine Intelligence Institute (Amii)}

\icmlaffiliation{cifar}{CIFAR AI Chair}

\icmlcorrespondingauthor{Mohamed Elsayed}{mohamedelsayed@ualberta.ca}
\icmlcorrespondingauthor{A. Rupam Mahmood}{armahmood@ualberta.ca}

\icmlkeywords{Machine Learning, ICML}

\vskip 0.3in
]




\printAffiliationsAndNotice{Preprint.\ Work in Progress.} 

\begin{abstract}
Modern representation learning methods often struggle to adapt quickly under non-stationarity because they suffer from catastrophic forgetting and decaying plasticity. Such problems prevent learners from fast adaptation since they may forget useful features or have difficulty learning new ones. Hence, these methods are rendered ineffective for continual learning. This paper proposes \emph{Utility-based Perturbed Gradient Descent} (UPGD), an online learning algorithm well-suited for continual learning agents. UPGD protects useful weights or features from forgetting and perturbs less useful ones based on their utilities. Our empirical results show that UPGD helps reduce forgetting and maintain plasticity, enabling modern representation learning methods to work effectively in continual learning.
\end{abstract}


\section{Introduction}
Learning online is crucial for systems that need to continually adapt to an ever-changing world. These learners perform updates as soon as the data arrives, allowing for fast adaptation. Such lifelong learners never stop learning and can run indefinitely. Thus, their required computation and memory should not grow when presented with more experiences. Moreover, it is desired if the learners are computationally cheap and maintain a small memory footprint, a criterion which we refer to as \emph{computational efficiency}. Therefore, desirable criteria for continual learning are fast online adaptation and computational efficiency.

The current gradient-descent methods prevent fast online adaption of continual learning systems. Such methods typically suffer from two main problems that slow down adaptation: decaying plasticity (Dohare et al.\ 2021) and catastrophic forgetting (McCloskey \& Cohen 1989). Decaying plasticity arises when the learner's ability to learn representations is hindered (e.g., when the number of saturated features increases, resulting in small gradients that obstruct fast changes to the weights). Catastrophic forgetting can be partly viewed as a result of shared representations, instead of sparse representations, in function approximation (French 1991, Liu et al.\ 2019). The weight updates are non-local---changes many weights---in the case of dense representations because many features are active for any given input, interfering with previously learned representations and potentially causing forgetting. Additionally, catastrophic forgetting is exacerbated by the inability of current gradient-descent methods to reuse previously learned useful features or protect them from change. It has been shown that current methods may destroy the most useful features under non-stationarity (Sutton 1986), making these methods de- and re-learn the features when the learners face similar or the same situations again.

Naturally, decaying plasticity and catastrophic forgetting co-occur in neural networks since neural networks use dense representations learned by gradient-based methods, and it might be hard to quickly change the functions they represent. Dense representation creates non-locality, resulting in forgetting, and many factors, such as vanishing gradients or correlated features, contribute to decaying plasticity.

Different approaches have been proposed to mitigate catastrophic forgetting. Typically, such methods are either replay-based methods (e.g., Chaudhry et al.\ 2019, Isele \& Cosgun 2018, Rolnick et al.\ 2019), regularization-based (e.g., Kirkpatrick et al.\ 2017, Aljundi et al.\ 2018, Aljundi et al.\ 2019), sparsity-inducing (e.g., Liu et al.\ 2019, Pan et al.\ 2021), or use dynamic architectures (e.g., Rusu et al.\ 2016, Schwarz et al.\ 2018). However, most methods are not able to satisfy our criteria of efficiency or fast online adaptation. Additionally, all methods addressing catastrophic forgetting still suffer from decaying plasticity, which can be mitigated by continual injection of noise (e.g., Ash \& Adams 2020, Zhou et al.\ 2019, Dohare et al.\ 2021,  Orvieto et al.\ 2022). However, such methods may still suffer from catastrophic forgetting.

Both problems of continual learning can be identified with the fact that gradient-descent methods are indifferent to how useful a feature is. When a feature becomes difficult to modify, contributing to decaying plasticity, the problem could be overcome by resetting or re-initializing that feature. Conversely, when a feature is useful and well-contributing to a task, it could be protected from further change and catastrophic forgetting. However, no existing gradient-descent methods contain such capabilities.

In this paper, we present Utility-based Perturbed Gradient Descent (UPGD), a novel online learning mechanism that protects useful weights or features and perturbs less useful ones based on their utilities. UPGD reuses useful features and builds new features based on existing ones, allowing fast adaptation. Moreover, UPGD estimates the utility of each weight or feature using linear computational complexity like SGD without storing past samples, making our method computationally scalable. Additionally, UPGD does not require the knowledge of task boundaries, making it a task-free continual learning method (Aljundi et al.\ 2019). 

The learning rule of UPGD contains two components helping in fast online adaptation: search and gradient. Each component is weighted by the utility estimates of the corresponding weights or features. The utility-informed search component helps find better weights and features through continual perturbation, which helps against decaying plasticity. The utility-based gradient is a modulated gradient signal---small for useful weights and large for less useful weights---that helps against catastrophic forgetting.


\section{Problem Formulation}
We use the online continual supervised learning setting where there is a stream of data examples. These data examples are generated from some non-stationary \emph{target function} ${f}_t$ mapping the input to the output, where the input-output pair at time $t$ is $(\vx_t, \vy_t)$. The \emph{learner} is required to predict the output given the input vector $\vx_t \in \R^d$ by estimating the target function ${f}_t$. We consider the case where ${f}_t$ is not arbitrarily non-stationary but rather locally stationary in time (e.g., slowly changing or can be divided into stationary tasks). We further consider target functions with regularities that can potentially reoccur. The performance is measured with some loss, $\mathcal{L}(\vy_t, \hat{\vy}_t)$, where $\vy_t\in \R^m$ is the target vector and $\hat{\vy}_t\in \mathbb{R}^m$ is the predicted output. Specifically, the mean squared error is used in regression, and cross-entropy is used in classification. The learner is required to reduce the loss by matching the target. The performance of the learner is measured based on the average online evaluation metric $\overline{E}$ over all time steps, which is given by
\begin{align}
\overline{E}(T) &=  \frac{1}{T}\sum_{t=1}^{T} E(\vy_t, \hat{\vy}_t), \label{equ:online-metric}
\end{align}
where $E$ is the sample evaluation metric (e.g., accuracy or loss), and $T$ is the total number of time steps.
We note here that this problem formulation has been introduced before by Caccia et al.\ (2020), where the performance is measured based on the online evaluation metric; however, it was assumed that we start with pre-trained learners. Here, we consider the setting where learners learn from scratch. Note that this online evaluation metric is similar to the cumulative sum of rewards in reinforcement learning.

Consider a neural network with $L$ layers that outputs the predicted output $\hat{\vy}$. The neural network is parametrized by the set of weights $\mathcal{W}=\{\mW_1,...,\mW_L\}$, where $\mW_l$ is the weight matrix at the $l$-th layer, and its element at the $i$th row and the $j$-th column is denoted by $W_{l,i,j}$. 
During learning, the parameters of the neural network are changed to reduce the loss. 
At each layer $l$, we get the activation output $\vh_{l}$ by applying the activation function $\boldsymbol\sigma$ to the activation input $\va_{l}$: $\vh_{l} = \boldsymbol\sigma(\va_{l})$.
We simplify notations by defining $\vh_0 \doteq \vx$.
The activation output $\vh_{l}$ is then multiplied by the weight matrix $\mW_{l+1}$ of layer $l+1$ to produce the next activation input: ${a}_{l+1,i} = \sum_{j=1}^{|\vh_{l}|} {{W}_{l+1,i,j}}{h}_{l,j}$. 
We assume here that the activation function is element-wise activation for all layers except for the final layer $L$ in classification tasks, where it becomes the softmax function.


\subsection{Decaying Plasticity}
Neural plasticity can be defined as the ability of neural networks to change in response to some stimulus (Konorski 1948, Hebb 1949). In artificial neural networks, we can think of plasticity as the ability to change the function represented by the neural network. Note that this definition is based on the ability to change the function, not the underlying weights. For example, a learner can update the weights in a constant-initialized neural network. Still, its ability to change the function is quite limited because of the feature symmetries, although the weights can be changed. We note here that the term plasticity is different from representation power or \emph{capacity}, which refers to the range of functions that can be represented by some architecture. In this paper, we use networks with the same capacity and show they lose plasticity when used with typical gradient-based methods.

Our online evaluation metric favors methods that adapt quickly to the changes since it is not a function of the final performance only. A plastic learner can adapt more quickly to changes than a less plastic one and therefore achieve higher cumulative average performance. Such an evaluation metric is natural to measure plasticity when the learner is presented with sequential tasks that require little transfer between them. Input-permuted tasks (e.g., permuted MNIST) satisfy this criterion since the representations learned in one task are somewhat irrelevant to the next task. When current gradient-based methods are presented with sequential input-permuted tasks, their performance degrades. Such degradation is partially explained by many hypotheses, such as increased numbers of saturated features incapable of fast adaptation (Dohare et al.\ 2021, Abbas et al.\ 2023) or the changing loss landscape under non-stationarity, making function adjustment harder (Lyle et al.\ 2023).


\subsection{Catastrophic Forgetting}

Catastrophic forgetting is the tendency of a neural network to forget past information after learning from new experiences. For example, when an agent learns two tasks sequentially, A and B, its ability to remember A would degrade catastrophically after learning B. We extend that definition to be concerned about learned features instead of tasks.

Forgetting is underexamined in the online setting where agents never stop learning. Our metric allows us to naturally have the relearning-based metric (Hetherington \& Seidenberg 1989) to measure forgetting by looking at how long it takes to relearn a new task again. This is suitable for online learning since we do not stop the agent from learning to evaluate its performance offline, which is required by the retention-based metric (McCloskey \& Cohen 1989).

Separating catastrophic forgetting from decaying plasticity is challenging since it is unclear how to create a network that does not lose plasticity. Most works addressing forgetting use networks that lose plasticity, which adds to performance degradation coming from forgetting. We also show that even linear neural networks (e.g., using identity activations) can lose plasticity too, which makes the problem of separating them more challenging.

Studying catastrophic forgetting can be done using a problem that requires transfer from one task to another. An ideal learner should be able to utilize previously learned useful features from one task to another. However, a typical gradient-based method (e.g., SGD) keeps re-learning useful features again and again because of forgetting, resulting in performance degradation.

In this paper, we present a method to reduce forgetting and maintain plasticity in a principled way by protecting useful weights or features and removing less useful ones.


\section{Method}
Estimating the utility of weights and features can help protect useful weights and features and get rid of less useful ones. The utility of a weight or a feature in a neural network used by a continual online learner from time step $t$ to time step $t+h$ can be defined as the averaged instantaneous utility over these time steps, which can be written as follows:
\begin{align*}
\mathcal{U}_{t:t+h} & \doteq \frac{1}{k} \sum_{k=0}^{h} \mathcal{U}_{t+k}.
\end{align*}
Since continual online systems run indefinitely, we are interested in the averaged instantaneous utility over all future time steps starting from time step $t$, which is given by 
\begin{align*}
\overline{\mathcal{U}}_t & \doteq \lim_{k \rightarrow \infty} \mathcal{U}_{t:t+k}.
\end{align*}
However, the average of future utilities is unknown at time step $t$. The learner can only compute the instantaneous utility $\mathcal{U}_{t}$. One can estimate such an unknown quantity by averaging over the past instantaneous utilities (e.g., utility traces), hoping that the past resembles the future. In our analysis, we focus on computing the instantaneous utility, but we use utility traces in our experiments.

The true instantaneous utility of a weight in a neural network can be defined as the change in the loss after removing that weight (Mozer \& Smolensky 1988, Karnin 1990). An important weight, when removed, should increase the loss. The utility matrix corresponding to $l$-th layer connections is $\mU_l(Z)$, where $Z=(X, Y)$ denotes the sample. We drop the time notation for ease of writing but emphasize the instantaneity of utility by writing it as a function of the sample. The true utility of the weight $i,j$ in the $l$-th layer for the sample $Z$ is defined as follows:
\begin{align}
U_{l,i,j}(Z) & \doteq \mathcal{L}(\mathcal{W}_{\neg[l,i,j]}, Z) - \mathcal{L}(\mathcal{W},Z),
\end{align}
where $\mathcal{L}(\mathcal{W}, Z)$ is the sample loss given the set of parameters $\mathcal{W}$, and $\mathcal{W}_{\neg[l,i,j]}$ is the same as $\mathcal{W}$ except the weight $W_{l,i,j}$ is set to 0.

Similarly, we define the true utility of a feature $i$ at layer $l$, which is represented as the change in the loss after the feature is removed. The utility of the feature $i$ in the $l$-th layer is given by
\begin{align}
u_{l,j}(Z) & \doteq \mathcal{L}(\mathcal{W}, Z | h_{l,j} = 0) - \mathcal{L}(\mathcal{W} , Z),
\end{align}
where $h_{l,j} = 0$ denotes setting the activation of the feature to zero (e.g., by adding a mask set to zero).

Note that both these utility measures are global measures, and they provide a total ordering for weights and features according to their importance. However, computing such a true utility is prohibitive since it requires additional $N_w$ or $N_f$ forward passes, where $N_w$ is the total number of weights and $N_f$ is the total number of features. Approximating the true utility helps reduce the computation needed to compute the utility such that no additional forward passes are needed.


\subsection{Approximated Weight Utility}
We approximate the true utility of weights by a second-order Taylor approximation. To write the utility $u_{l,i,j}$ of the connection $ij$ at layer $l$, we expand the true utility around the current weight of the connection $ij$ at layer $l$ and evaluate it at the value of that weight being zero. The quadratic approximation of $U_{l,i,j}(Z)$ can be written as
\begin{align*}
    U_{l,i,j}(Z) &= \mathcal{L}(\mathcal{W}_{\neg[l,i,j]}, Z) - \mathcal{L}(\mathcal{W} , Z)\\
    &\approx \mathcal{L}(\mathcal{W} , Z) + \frac{\partial \mathcal{L}(\mathcal{W}, Z)}{\partial W_{l,i,j}}(0 - W_{l,i,j}) \\ 
    & \quad +\frac{1}{2}\frac{\partial^2 \mathcal{L}}{\partial W^2_{l,ij}} (0-W_{l,ij})^2 - \mathcal{L}(\mathcal{W} , Z) \\
    &= -\frac{\partial \mathcal{L}(\mathcal{W}, Z)}{\partial W_{l,i,j}} W_{l,i,j} + \frac{1}{2}\frac{\partial^2 \mathcal{L}(\mathcal{W}, Z)}{\partial W^2_{l,ij}} W^2_{l,ij}.
\end{align*}
We refer to the utility measure containing the first term as the \textit{first-order approximated weight utility}, and the utility measure containing both terms as the \textit{second-order approximated weight utility}. The computation required for the second-order term has quadratic complexity. Therefore, we further approximate it using HesScale (Elsayed \& Mahmood 2022; see Appendix \ref{appendix:hesscale}), making the computation of both of these approximations have linear complexity. Moreover, we present a way for propagating our approximated utilities by the \textit{utility propagation theorem} in Appendix \ref{appendix:utility-propagation}.

\subsection{Approximated Feature Utility}
Here, we derive the global utility of a feature $j$ at layer $l$, which is represented as the difference between the loss when the feature is removed and the original loss. To have an easier derivation, we add a mask (or gate) on top of the activation output: $\bar{\vh_{l}} = \vm_{l} \circ \vh_{l}$. Note that the weights of such masks are set to ones and never change throughout learning. The quadratic approximation of $u_{l,j}(Z)$ can be written as
\begin{align*}
u_{l,j}(Z) &= \mathcal{L}(\mathcal{W}, Z | m_{l,j} = 0) - \mathcal{L}(\mathcal{W} , Z) \nonumber \\
&\approx \mathcal{L}(\mathcal{W} , Z) + \frac{\partial \mathcal{L}}{\partial m_{l,i}}(0 - m_{l,j}) \\
& \quad + \frac{1}{2}\frac{\partial^2 \mathcal{L}}{\partial m^2_{l,i}} (0 - m_{l,j})^2 \nonumber - \mathcal{L}(\mathcal{W} , Z)\\
&= -\frac{\partial \mathcal{L}}{\partial m_{l,i}}+ \frac{1}{2}\frac{\partial^2 \mathcal{L}}{\partial m^2_{l,i}}.
\end{align*}

We further approximate it using HesScale, making the computation have linear computational complexity.

When we use an origin-passing activation function, we can instead expand the loss around the current activation input $a_{l,i}$ without the need to use a mask. Note that derivatives with respect to the activation inputs are available from backpropagation. The feature utility can be simply:
\begin{align}
u_{l,j}(Z) &= \mathcal{L}(\mathcal{W}, Z | a_{l,j} = 0) - \mathcal{L}(\mathcal{W} , Z) \nonumber \\
&\approx  -\frac{\partial \mathcal{L}}{\partial a_{l,i}} a_{l,i} + \frac{1}{2}\frac{\partial^2 \mathcal{L}}{\partial a^2_{l,i}} a^2_{l,i}. \nonumber
\end{align}
Moreover, the approximated feature utility can be computed using the approximated weight utility, which gives rise to the \textit{conservation of utility} property. We show such a relationship and its proof in Appendix \ref{appendix:feature-weight-connection}.


\subsection{Utility-based Perturbed Search}
Searching in the weight space or the feature space help find better weights and features. We devise a novel learning rule based on utility-informed search. Specifically, we protect useful weights or features and perturb less useful ones. We show the rule of Utility-based Perturbed Search (UPS) to update the weights. We emphasize here that utility-based perturbed search does not move in the same direction as the gradient; hence, it is not a gradient-descent update. The update rule of UPS is defined as follows:
\begin{align}
w_{l, i,j} &\leftarrow w_{l, i,j} - \alpha\xi(1-\Bar{U}_{l,i,j}),
\end{align}
where $\xi$ is a noise sample, $\alpha$ is the step size, and $\Bar{U}_{l,i,j}$ is a scaled utility with a minimum of $0$ and a maximum of $1$. Such scaling helps protect useful weights and perturb less useful ones. For example, a weight with $\Bar{U}_{l,i,j}=1$ is not perturbed, whereas a weight with $\Bar{U}_{l,i,j}=0$ is perturbed by the whole noise sample. The noise type and amount play an important role in search. Moreover, search difficulty is a function of the search space size, meaning that we can expect searching in the feature space is easier than in the weight space. We can retrieve the weight-wise and feature-wise UPS algorithms by dropping the gradient information from the update equation in Algorithm \ref{alg:upgd-weight-global} and Algorithm \ref{alg:upgd-feature-global}. UPS learning rule resembles an evolutionary process on the weights or the features where each time step resembles a new population of weights or features, and the scaled utility resembles the fitness function that either keeps the useful (fittest) weights or features or perturbs the less useful ones.


\subsection{Utility-based Perturbed Gradient Descent}
Our aim is to write an update equation to protect useful weights and replace less useful ones using search and gradients.  We show here the rule of Utility-based Perturbed Gradient Descent (UPGD) to update the weights. The update equation is given by
\begin{align}
    w_{l, i,j} \leftarrow w_{l, i,j} - \alpha\left(\frac{\partial \mathcal{L}}{\partial w_{l, i,j}} + \xi \right) \left(1-\Bar{U}_{l,i,j}\right).\label{equ:upgd}
\end{align}

The utility information in UPGD works as a gate for perturbed gradients. For important weights with utility $\bar{U}_{l,i,j}=1$, the weight is not updated, whereas unimportant weights with utility $\bar{U}_{l,i,j}=0$ get updated by the whole perturbed gradient information. We note here that UPGD is able to introduce plasticity by perturbing less useful weights and reduce forgetting by protecting useful weights.

Another variation of UPGD, we call \textit{non-protecting UPGD}, is to add the utility-based perturbation to the gradient as:
\begin{align}
    w_{l, i,j} \leftarrow w_{l, i,j} - \alpha\left[\frac{\partial \mathcal{L}}{\partial w_{l, i,j}} + \xi (1-\Bar{U}_{l,i,j})\right]. \label{equ:nonprotecting-upgd}
\end{align}
However, such an update rule can only help against decaying plasticity, not catastrophic forgetting. This is because useful weights are not protected from change by gradients.

Utility scaling is important for the UPGD and UPS update equations. We devise two kinds of scaling: global and local. The global scaled utility requires the maximum utility of all weights or features at every time step. The scaled utility for the global-utility UPGD and UPS is given by $\Bar{U}_{l,i,j} = \phi({U}_{l,i,j} / \eta)$ for weights, where $\eta$ is the maximum utility of the weights and $\phi$ is the scaling function (e.g., Sigmoid). For feature-wise UPGD and UPS, the scaled utility is given by $\Bar{u}_{l,j} = \phi({u}_{l,j} / \eta)$. We show the pseudo-code of our method using the global scaled utility in Algorithm \ref{alg:upgd-weight-global} for weight-wise UPGD and in Algorithm \ref{alg:upgd-feature-global} for feature-wise UPGD. The local UPGD and UPS do not require the global max operation since it normalizes the outgoing weight vector for each feature. Specifically, the scaled utility is given by $\Bar{U}_{l,i,j} = \phi\big({U}_{l,i,j}/\sqrt{\sum_{j}{U}^2_{l,i,j}}\big)$ for weight-wise UPGD and UPS. For features, the scaled utility is given by $\Bar{u}_{l,j} = \phi\big({u}_{l,j}/\sqrt{\sum_{j}{u}^2_{l,j}}\big)$. We show the pseudo-code of our method using the local scaled utility in Algorithm \ref{alg:upgd-weight-layerwise} for weight-wise UPGD and in Algorithm \ref{alg:upgd-feature-layerwise} for feature-wise UPGD.


\subsection{Perturbed Gradient Descent Methods}
One can perturb all weights evenly, which can give rise to a well-known class of algorithms called Perturbed Gradient Descent (PGD). The learning rule of PGD is given by:
\begin{align}
    w_{l, i,j} \leftarrow w_{l, i,j} - \alpha\left[\frac{\partial \mathcal{L}}{\partial w_{l, i,j}} + \xi \right].\label{equ:pgd}
\end{align}
where $\xi$ is a noise perturbation that can be uncorrelated, giving PGD (Zhou et al.\ 2019), or anti-uncorrelated, giving Anti PGD (Orvieto et al.\ 2022). Anti-correlated perturbation at time $t+1$ is given by $\xi_{t+1} = \zeta_{t+1} - \zeta_{t}$, where $\zeta_{t+1}$ and $\zeta_{t}$ are perturbations sampled from $\mathcal{N}(0,1)$, whereas uncorrelated perturbation at time $t+1$ is given by $\xi_{t+1} = \zeta_{t+1}$. In this paper, we care about online methods, so we only consider stochastic versions of these methods. When all scaled utilities are zeros in UPGD (see Eq.\ \ref{equ:upgd} and Eq.\ \ref{equ:nonprotecting-upgd}), UPGD reduces to PGD methods.

Many works have shown the role of noise in improving generalization and avoiding bad minima (e.g., PGD, Anti-PGD) in stationary problems. A stochastic version of perturbed gradient descent (Neelakantan et al.\ 2015) helps escape bad minima and improve performance and generalization.

It has not been shown before that these PGD methods can help improve plasticity; however, recent work by Dohare et al.\ (2021) showed that noise injection by resetting features helps maintain plasticity. We hypothesize that stochastic perturbed gradient descent methods might help too.

On the other hand, it has been shown that a stochastic PGD with a weight decay algorithm, known as \emph{Shrink and Perturb} (Ash \& Adams 2020), can help maintain plasticity in continual classification problems (Dohare et al.\ 2022). The learning rule of shrink and perturb can be written as
\begin{align}
     w_{l, i,j} \leftarrow \rho w_{l, i,j} - \alpha\left[\frac{\partial \mathcal{L}}{\partial w_{l, i,j}} + \xi \right].
\end{align}
where $\rho =1-\lambda\alpha$ and $\lambda$ is the weight decay factor. When no noise is added, the update reduces to SGD with weight decay (Loshchilov \& Hutter 2019), known as SGDW.

\subsection{Utility-based Perturbed Gradient Descent with Weight Decay}
Combining weight decay with our UPGD learning rules, one can write the UPGD with weight decay as follows:
\begin{align}
    w_{l, i,j} \leftarrow \rho w_{l, i,j} - \alpha\left(\frac{\partial \mathcal{L}}{\partial w_{l, i,j}} + \xi \right) \left(1-\Bar{U}_{l,i,j}\right).
\end{align}

Similarly, the non-protecting UPGD learning rule with weight decay can be written as
\begin{align}
    w_{l, i,j} \leftarrow \rho w_{l, i,j} - \alpha\left[\frac{\partial \mathcal{L}}{\partial w_{l, i,j}} + \xi (1-\Bar{U}_{l,i,j})\right].
\end{align}

\begin{algorithm}[ht]
\caption{Weight-wise UPGD with global scaled utility}\label{alg:upgd-weight-global}
\begin{algorithmic}
\STATE {\bfseries Require:} Neural network $f$ with weights $\{\mW_1,...,\mW_L\}$ and a stream of data $\mathcal{D}$.
\STATE {\bfseries Require:} Step size $\alpha$ and decay rate $\beta \in [0,1)$.
\STATE {\bfseries Require:} Initialize $\{\mW_1,...,\mW_L\}$
\STATE {\bfseries Require:} Initialize time step $t\leftarrow0$; $\eta\leftarrow-\infty$
\FOR{$l$ in $\{L,L-1,...,1\}$} 
\STATE $\mU_{l} \leftarrow \mathbf{0}$; $\hat{\mU}_{l} \leftarrow \mathbf{0}$
\ENDFOR
\FOR{$(\vx, \vy)$ in $\mathcal{D}$}
\STATE $t\leftarrow t+1$
\FOR{$l$ in $\{L,L-1,...,1\}$}
\STATE $\mF_{l},\mS_{l}\leftarrow$\texttt{GetWeightDerivatives}($f, \vx, \vy, l)$
\STATE $\mM_{l} \leftarrow \sfrac{1}{2}\mS_{l}\circ\mW^2_{l} - \mF_{l}\circ\mW $
\STATE $\mU_{l} \leftarrow \beta\mU_{l} + (1-\beta)\mM_l $
\STATE $\hat{\mU}_{l} \leftarrow \mU_{l}/(1-\beta^{t})$
\IF{$\eta < \texttt{max}(\hat{\mU}_{l})$}
\STATE $\eta \leftarrow \texttt{max}(\hat{\mU}_{l})$
\ENDIF
\ENDFOR
\FOR{$l$ in $\{L,L-1,...,1\}$}
\STATE Sample noise matrix $\boldsymbol{\xi}$
\STATE $\Bar{\mU}_{l} \leftarrow \phi{(\hat{\mU_{l}}/\eta)}$
\STATE $\mW_{l} \leftarrow \mW_{l} - \alpha (\mF_{l} + \boldsymbol{\xi})\circ(1-\Bar{\mU}_{l})$
\ENDFOR
\ENDFOR
\end{algorithmic}
\end{algorithm}


\section{Experiments}

\label{section:experiments}
In this section, we design and perform a series of experiments to estimate the quality of our approximated weight or feature utilities. In addition, we showcase the effectiveness of searching in weight or feature space. We then evaluate weight-wise and feature-wise UPGD in mitigating decaying plasticity and catastrophic forgetting issues using non-stationary problems based on MNIST (LeCun et al.\ 1998), EMNIST (Cohen et al.\ 2017), and CIFAR-10 (Krizhevsky 2009) datasets.

Although UPS and UPGD can be used in principle in the offline learning setting, the focus of our paper is continual online learning. The performance of continual learners is evaluated using the metric in Eq.\ \ref{equ:online-metric}, where we use the average online loss for regression tasks and average online accuracy for classification tasks.

In each of the following experiments, a hyperparameter search is conducted. Our criterion was to find the best set of hyperparameters for each method in that search space which minimizes the area under the loss curve (in regression tasks) or maximizes the area under the accuracy curve (in classification tasks). Unless stated otherwise, we averaged the performance of each method over $20$ independent runs.


\subsection{Quality of Approximated Utility Ordering}
A high-quality approximation should give a similar ordering of weights or features. We use the ordinal correlation measure of Spearman to quantify the quality of our utility approximations. We start by showing the correlation between different approximated weight utilities and the true utility. A network of a single hidden layer containing $50$ units is used. The network has five inputs and a single output. The target of an input vector is the sum of two inputs out of the five inputs. The inputs are sampled from $U[-0.5, 0.5]$. The weights are initialized by Kaiming initialization (He et al.\ 2015). SGD is used to optimize the parameters of the network to minimize the mean squared error with a total of $2000$ samples. At each time step, the Spearman correlation metric is calculated for first-order and second-order approximated global weight utility against the random utility and the weight-magnitude utility. The Spearman correlation is measured per sample based on $5\times 50+50+50+1=351$ items coming from weights and biases. We report the global correlations between the true utility and approximated weight utilities in Fig.\ \ref{fig:global-weight-utility} and defer the local correlations to Appendix \ref{appendix:utility-additional}. We use ReLU (Nair \& Hinton 2010) activations here and report the results with Tanh and LeakyReLU (Maas et al.\ 2013) in Appendix \ref{appendix:utility-additional}.

The correlation is the highest for the second-order approximated utility during learning up to convergence. On the other hand, the first-order approximated utility becomes less correlated when the learner converges since the gradients become very small. The weight-magnitude utility shows a small correlation to the true utility.  We use the random utility as a reference which maintains zero correlation with the true utility, as expected.

\begin{figure}[ht]
\centering
\subfigure[Weight-wise]{
    \includegraphics[width=0.46\columnwidth]{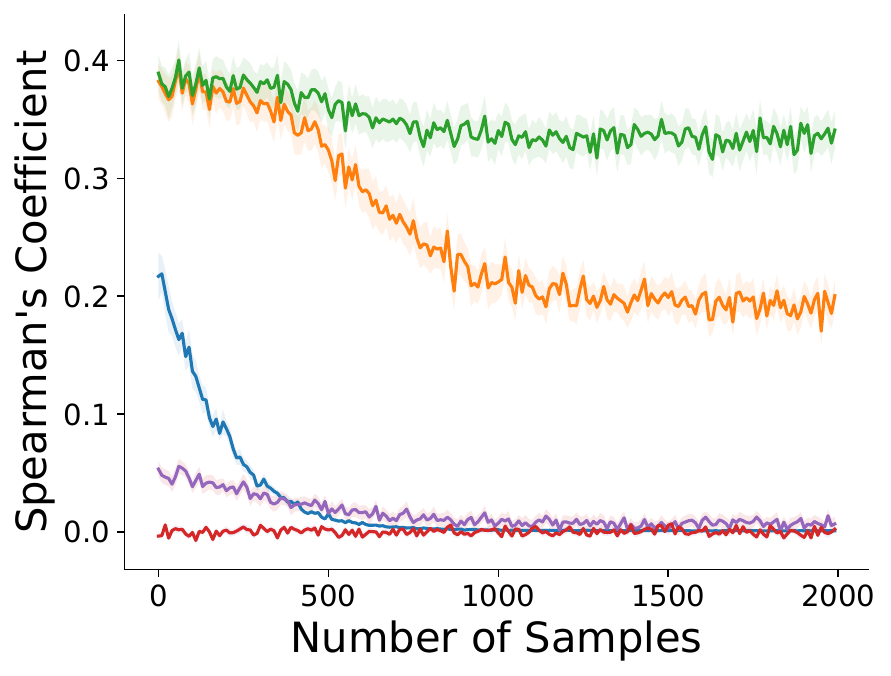}
    \label{fig:global-weight-utility}
}
\subfigure[Feature-wise]{
    \includegraphics[width=0.46\columnwidth]{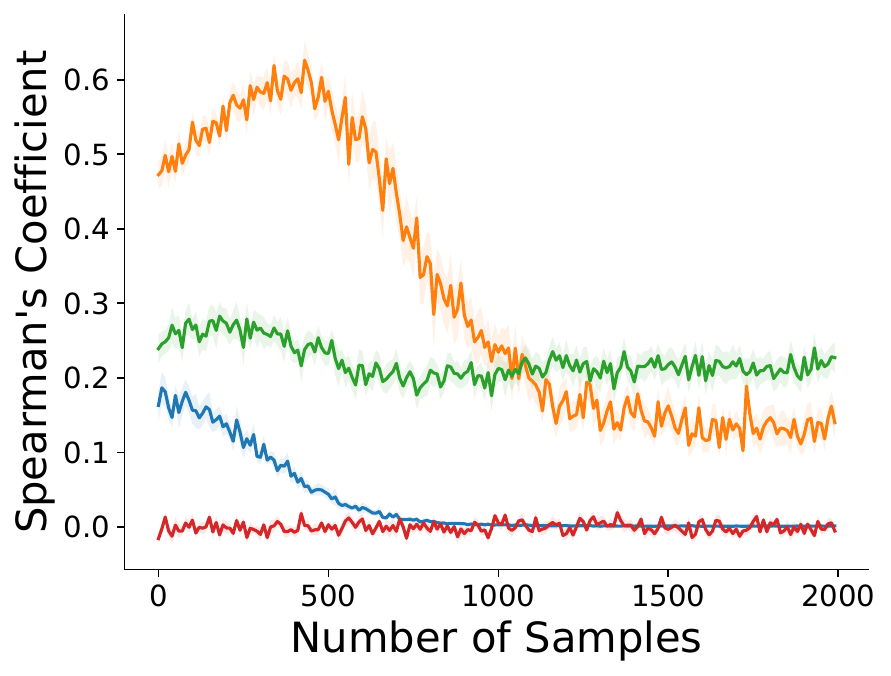}
    \label{fig:global-feature-utility}
}
\includegraphics[width=0.95\columnwidth]{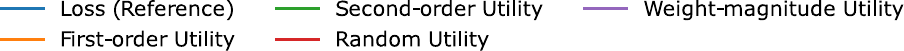}
\vspace{-0.2cm}
\caption{Spearman rank correlation between the true utility and global approximated utilities.}
\label{fig:global-utility}
\vspace{-0.2cm}
\end{figure}

We repeat the experiment to compute the approximated global feature utilities against the true utility and report the results in Fig.\ \ref{fig:global-feature-utility}. The network used has two hidden layers, each containing 50 units. The Spearman correlation is measured at each sample based on $50\times2=100$ items coming from features. We use ReLU activations here, and we report the results with Tanh and LeakyReLU in Appendix \ref{appendix:utility-additional} in addition to the local correlations. The results are similar to the weight-wise results. However, the approximated first-order and second-order both have a low correlation. This result suggests that second-order Taylor's approximation may not be sufficient to approximate the true feature utility. For that reason, We can expect that the weight-wise UPGD methods would perform better than the feature-wise ones.


\subsection{Utility-based Search Performance on MNIST}
In this experiment, we evaluate our utility-based search method in minimizing a loss using utility-informed search. We use MNIST as our stationary task. The learner is trained online for a 1M time step where only one input-output pair is presented at each time step. A network of two hidden layers, the first has 300 units, and the second has 150 units is used. The weights are initialized by Kaiming initialization. We report the results in Fig.\ \ref{fig:mnist-search} with ReLU activations. The type of noise used controls the performance of search optimizers. We noticed that the anti-correlated noise helps in first-order and second-order approximated utilities more than uncorrelated noise (shown in Appendix \ref{appendix:additional-mnist-ups}). The results show the effectiveness of utility-based search in reducing the loss against the baseline that perturbs all weights or features evenly (e.g., uses random utility). We note that utility-informed search alone, although being able to reduce the loss, is not as effective as using gradient information. When UPGD is used with MNIST, we notice an improvement in performance compared to SGD (see Appendix \ref{appendix:additional-mnist-upgd}).

\begin{figure}[ht]
\centering
\subfigure[Weight-wise]{
    \includegraphics[width=0.46\columnwidth]{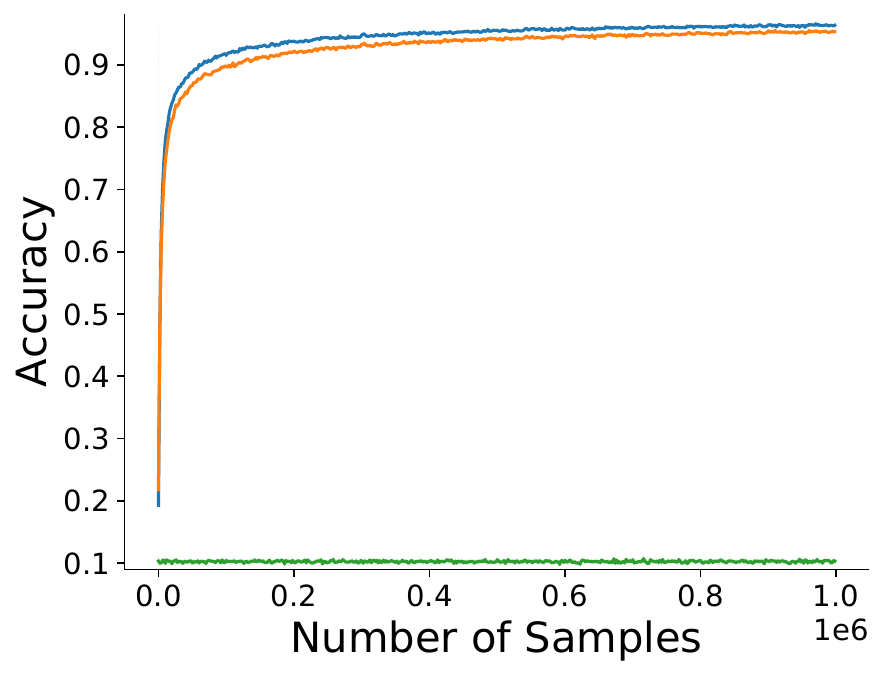}
    \label{fig:mnist-weight-search-anticorr}
}
\subfigure[Feature-wise]{
    \includegraphics[width=0.46\columnwidth]{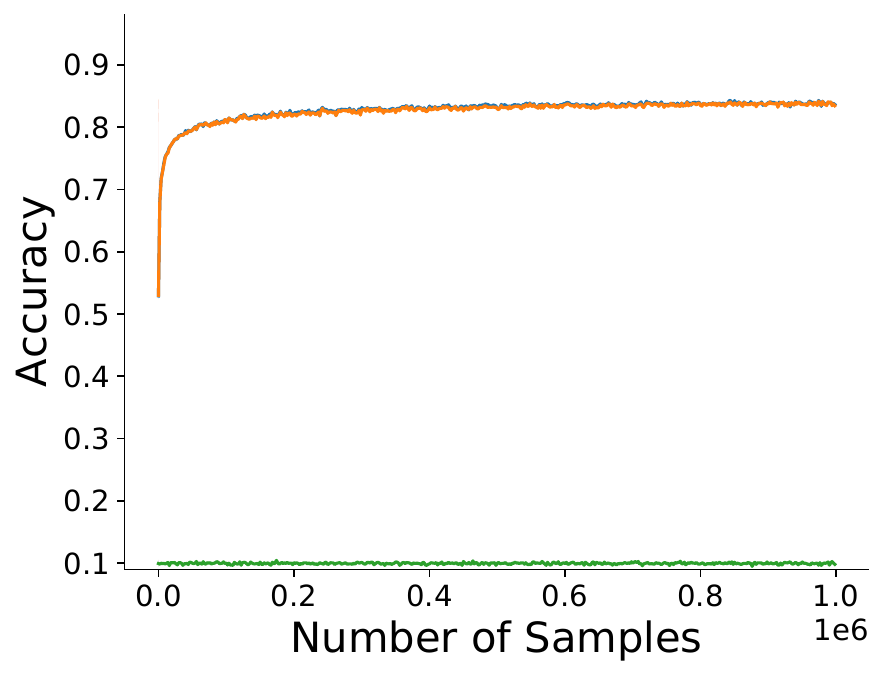}
    \label{fig:mnist-feature-search-anticorr}
}
\includegraphics[width=0.9\columnwidth]{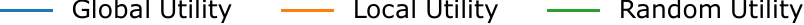}
\vspace{-0.2cm}
\caption{Performance of Utility-based Search on MNIST with first-order global and local utilities against a random utility.}
\label{fig:mnist-search}
\vspace{-0.2cm}
\end{figure}


\subsection{UPGD Performance on Nonstationary Toy Problem}
We present a simple regression problem for which the target at time $t$ is given by $y_t = a \sum_{i\in \mathcal{S}} x_{t,i}$, where $x_{t,i}$ is the $i$th entry of input vector at time $t$, $\mathcal{S}$ is the input set, and $a\in \mathbb{R}$. We introduce non-stationarity using two ways: changing the multiplier $a$ or changing the input set $\mathcal{S}$. In this problem, the task is to add two inputs out of 16 inputs. The learner is required to match the targets by minimizing the mean-squared error. The learner uses a multi-layer linear network that has two hidden layers containing $300$ and $150$ units, respectively. The network is linear since the activation used is the identity activation ($\sigma(\vx)=\vx$). Here, we show the results and give the experimental details in Appendix \ref{appendix:details-toy}.

\begin{figure}[ht]
\centering
\subfigure[Weight-wise]{
    \includegraphics[width=0.46\columnwidth]{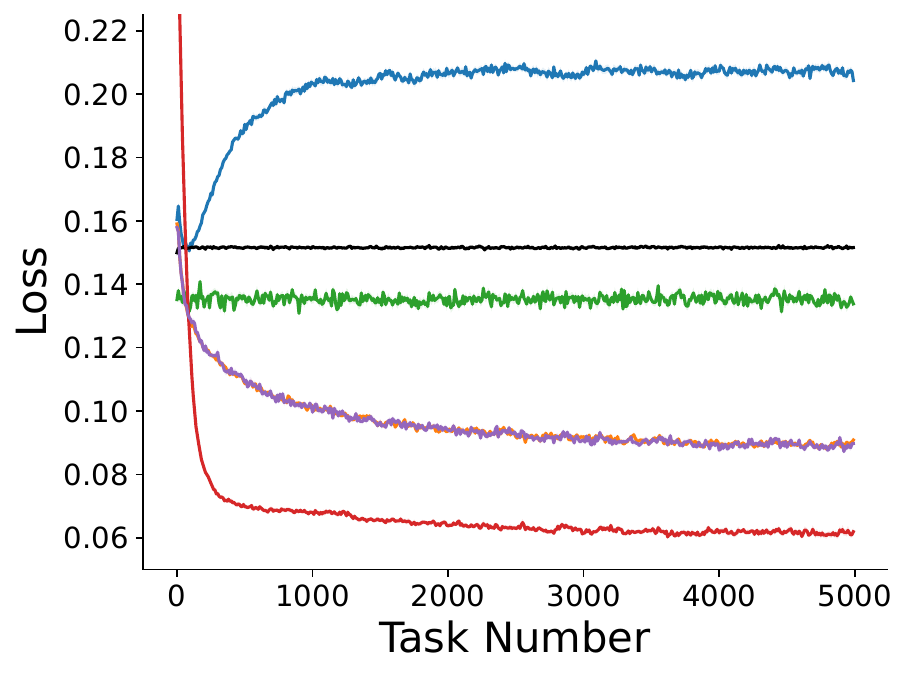}
    \label{fig:ChangingAvgInputs-FO-global-weight}
 }
\subfigure[Feature-wise]{
    \includegraphics[width=0.46\columnwidth]{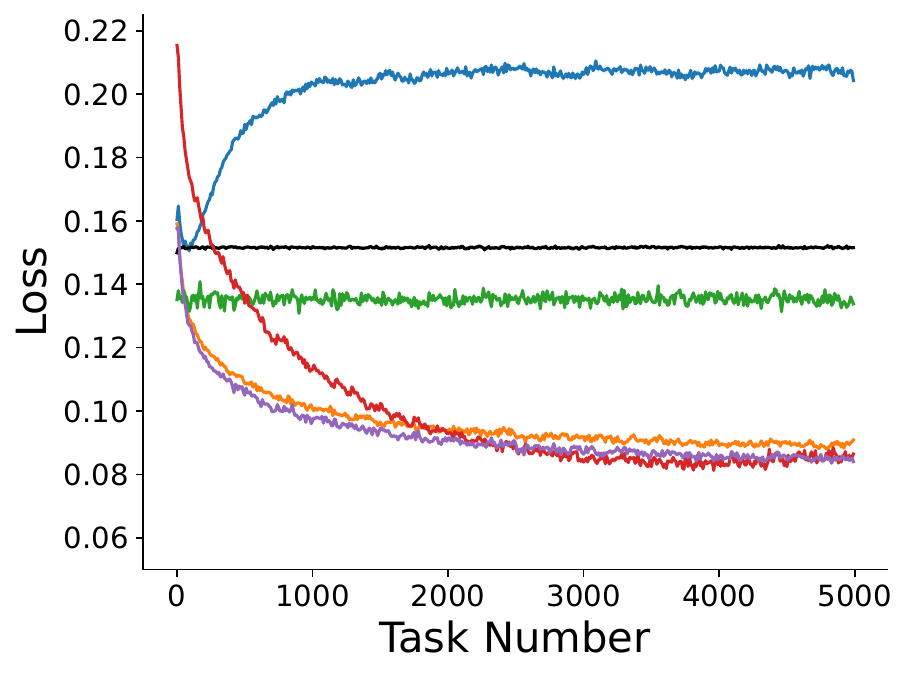}
    \label{fig:ChangingAvgInputs-FO-global-feature}
}
\includegraphics[width=0.9\columnwidth]{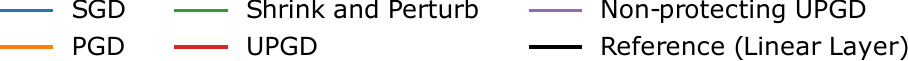}
\vspace{-0.2cm}
\caption{Performance of UPGD on the toy problem with a changing input set using the first-order approximated utility against SGD, PGD, and Shrink \& Perturb.}
\label{fig:ChangingAvgInputs-FO-global}
\vspace{-0.2cm}
\end{figure}

In the first variation of the problem, the input set $\mathcal{S}$ has a size of two, but the elements change every $200$ time steps by a shift of two in the input indices. For example, if the first task has $\mathcal{S}=\{1, 2\}$, the next would be $\{3, 4\}$ and so on. Since the tasks have little transfer between them, we expect the continual learners to learn as quickly as possible and maintain their plasticity. We compare UPGD against SGD, PGD, Shrink \& Perturb, and Non-protecting UPGD. Moreover, we use a baseline that has one linear layer mapping the input to the output. Note that we use first-order approximated utility here and defer the results with the second-order one in Appendix \ref{appendix:so-utility-toy-problem}. We report the results of this experiment in Fig.\ \ref{fig:ChangingAvgInputs-FO-global-weight} for weight-wise UPGD and in Fig.\ \ref{fig:ChangingAvgInputs-FO-global-feature} for feature-wise UPGD. The performance of SGD degrades with changing targets, indicating SGD loses plasticity every time the targets change. This may suggest that the outgoing weights to some features get smaller, hindering the ability to change the features' input weights. On the other hand, Shrink \& Perturb can maintain some plasticity over the linear-layer baseline. PGD and Non-protecting UPGD perform better than Shrink \& Perturb, indicating that weight decay is not helpful in this problem, and it is better to just inject noise without shrinking the parameters. UPGD is able to maintain its plasticity. Moreover, the performance keeps improving with changing targets compared to other methods.

\begin{figure}[ht]
\centering
\subfigure[Weight-wise]{
    \includegraphics[width=0.46\columnwidth]{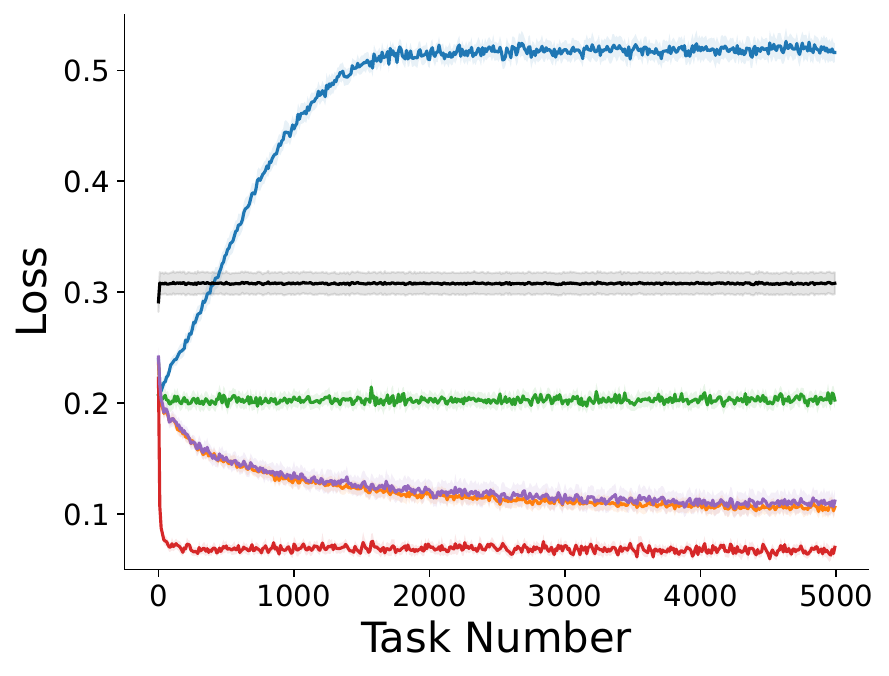}
    \label{fig:ChangingAvgSign-FO-global-weight}
 }
\subfigure[Feature-wise]{
    \includegraphics[width=0.46\columnwidth]{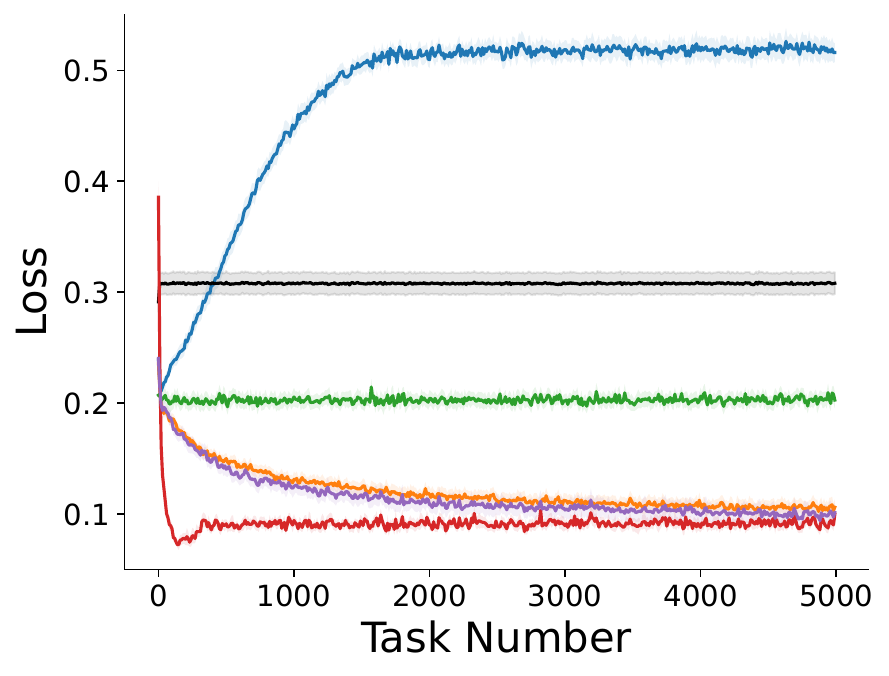}
    \label{fig:ChangingAvgSign-FO-global-feature}
}
\includegraphics[width=0.9\columnwidth]{figures/Legends/main_legend.pdf}
\vspace{-0.2cm}
\caption{Performance of UPGD on the toy problem with changing outputs using the first-order approximated utility against SGD, PGD, and Shrink \& Perturb.}
\label{fig:ChangingAvgSign-FO-global}
\vspace{-0.2cm}
\end{figure}
In the second variation, the sign of the target sum is flipped every $200$ time steps by changing $a$ from $1$ to $-1$ and vice versa. We expect continual learning agents to learn some features during the first $200$ steps. After the target sign change, we expect the learner to change the sign of only the output weights since the learned features should be the same. The frequency of changing $a$ is high to punish learners for re-learning features from scratch. Note that we use first-order approximated utility here and defer the results with second-order approximated utility in Appendix \ref{appendix:so-utility-toy-problem}. We report the results of this experiment in Fig.\ \ref{fig:ChangingAvgSign-FO-global-weight} for weight-wise UPGD and in Fig.\ \ref{fig:ChangingAvgSign-FO-global-feature} for feature-wise UPGD. The performance of SGD degrades with changing targets, indicating that it does not utilize learned features and re-learn them every time the targets change. Shrink \& Perturb, Non-protecting UPGD, and PGD maintain plasticity, but they are not able to protect useful weights; therefore, their performance is worse than UPGD. UPGD is able to protect useful weights or features and utilize them each time the targets change. Moreover, the performance keeps improving with changing targets compared to the other methods.


\subsection{UPGD Performance on Input-Permuted MNIST}
We can study plasticity when the learner is presented with sequential tasks that require little transfer between them. Input-permuted MNIST satisfies this criterion since the representations learned in one task are not relevant to the other tasks. We permute the inputs every $5000$ time steps and present the learners with $1$ million examples, one example per time step. The learner is required to maximize online accuracy by matching the target. The learner uses a multi-layer linear network with ReLU activations that has two hidden layers containing $300$ and $150$ units, respectively. 

\begin{figure}[ht]
\centering
\includegraphics[width=0.7\columnwidth]{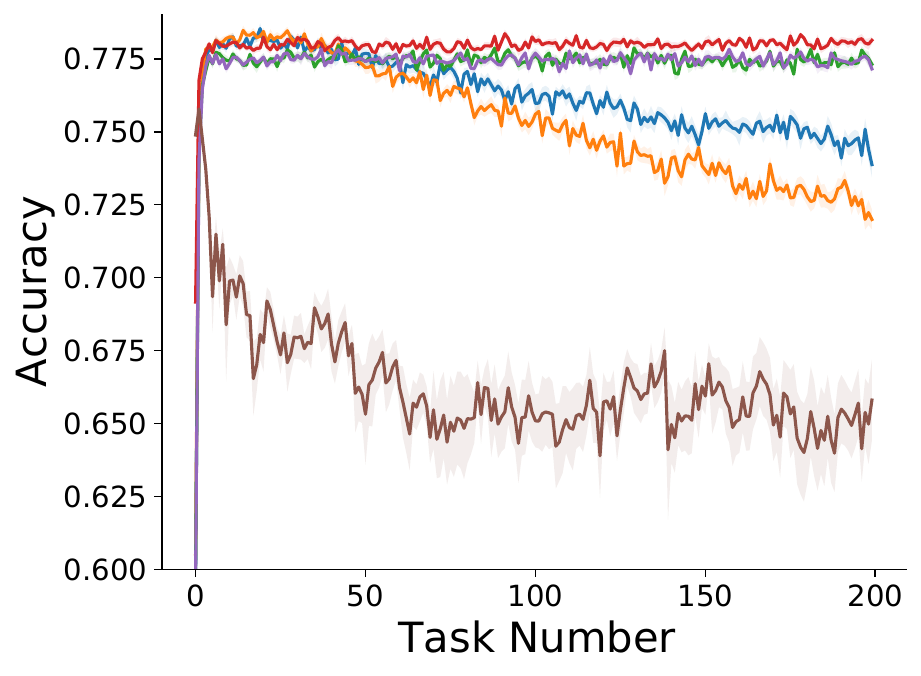}
\vspace{0.2cm}

\includegraphics[width=0.95\columnwidth]{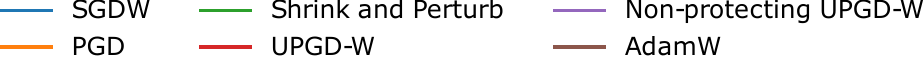}
\vspace{-0.2cm}
\caption{Performance of UPGD, SGD, and Adam with weight decay, PGD and Shrink \& Perturb on Input-Permuted MNIST. A global first-order utility is used with the two UPGDs.}
\label{fig:inputpermuted-mnist-upgd-global}
\vspace{-0.2cm}
\end{figure}

Dohare et al.\ (2022) have shown that weight decay helps in input-permuted MNIST and suggested that weight decay promotes plastic weights since maintaining small weights prevents weights from overcommitting, making them easy to change. Therefore, we compare SGD with weight decay (SGDW), Shrink \& Perturb, Adam with weight decay (Loshchilov \& Hutter 2019) known as (AdamW), UPGD with weight decay (UPGD-W) and Non-protecting UPGD-W. We note here that we do not compare with the method of Continual Backprop (2021), which maintain plasticity since it requires a complex learning mechanism, not a simple update rule. Additionally, Continual Backprop has a similar performance to Shrink \& Perturb (Dohare et al.\ 2022). Fig.\ \ref{fig:inputpermuted-mnist-upgd-global} shows the average online accuracy with the task number. All methods are tuned extensively (shown in Appendix \ref{appendix:details-input-permuted-mnist}). When more tasks are presented, the online accuracy of PGD, SGDW, and AdamW degrades with time. Only UPGD-W, Non-protecting UPGD-W, and Shrink \& Perturb could maintain plasticity. We also note that UPGD-W has slightly better performance than Non-protecting UPGD-W and Shrink \& Perturb. In Appendix \ref{appendix:ablation-study}, we present an ablation study for the effect of weight decay and perturbation on performance in Shrink \& Perturb, UPGD-W, and Non-protecting UPGD-W. We repeat this experiment for feature-wise UPGD in Appendix \ref{appendix:additional-input-permuted-mnist}.


\subsection{UPGD Performance on Output-permuted EMNIST}

Here, we study the interplay of catastrophic forgetting and decaying plasticity with Output-permuted EMNIST. This adds one level of complexity on top of the decaying plasticity shown in the previous section. The EMNIST dataset is an extended form of MNIST that has $47$ classes and has both digits and letters. We chose EMNIST instead of MNIST since Ouput-permuted MNIST was a very easy problem such that all methods achieved very high accuracy. In our Output-permuted EMNIST, the labels are permuted every $2500$ time steps. Such a change should not make the agent change its learned representations since it can simply change the weights of the last layer to adapt to that change. This makes the Output-permuted MNIST task suitable for studying catastrophic forgetting and decaying plasticity.

\begin{figure}[ht]
\centering
\includegraphics[width=0.7\columnwidth]{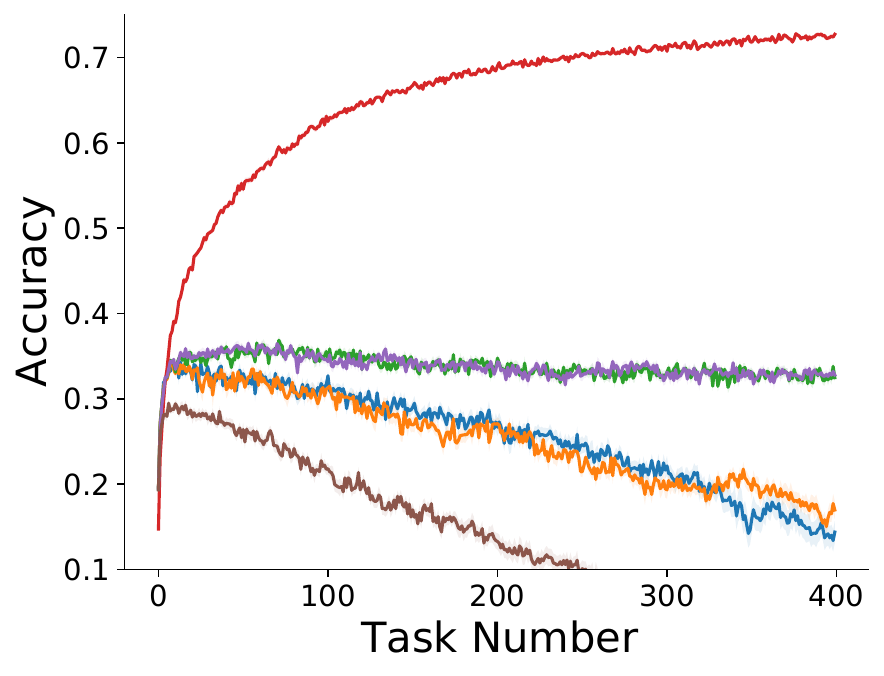}

\vspace{0.2cm}
\includegraphics[width=0.95\columnwidth]{figures/Legends/test_legend.pdf}
\vspace{-0.2cm}
\caption{Performance of UPGD, SGD, and Adam with weight decay, PGD and Shrink \& Perturb on Output-Permuted EMNIST. A global first-order utility is used with the two UPGDs.}
\label{fig:labelpermuted-emnist-upgd-global}
\vspace{-0.2cm}
\end{figure}

We compare SGDW, Shrink \& Perturb, AdamW, UPGD-W, and Non-protecting UPGD-W. All methods are tuned extensively (shown in Appendix \ref{appendix:details-output-permuted-emnist}). Fig.\ \ref{fig:labelpermuted-emnist-upgd-global} shows the average online accuracy with the task number. When more tasks are presented, the online accuracy of PGD, SGDW, and AdamW degrades with time. Non-protecting UPGD-W and Shrink \& Perturb could maintain their plasticity, but they slowly lose it as their online accuracy gradually goes down. In contrast, UPGD-W has significantly better performance than other algorithms. This suggests that UPGD-W maintains plasticity and reduces forgetting such that at every task, it can improve its representations. We found that UPGD-W performs the best when it has no weight decay in this task. In Appendix \ref{appendix:ablation-study}, we present an ablation study for the effect of weight decay and perturbation on performance. We repeat this experiment for feature-wise UPGD in Appendix \ref{appendix:additional-output-permuted-emnist}.


\subsection{UPGD Performance on Output-permuted CIFAR10}

Here, we study the interplay of catastrophic forgetting and decaying plasticity with Output-permuted CIFAR-10. The labels are permuted every $2500$ time step. Such a change should not make the agent change its learned representations since it can simply change the weights of the last layer to adapt to that change. In this problem, we make learners use a network with two convolutional layers with max-pooling followed by two fully connected layers with ReLU activations. Here, we show the results and give the experimental details in Appendix \ref{appendix:details-output-permuted-cifar10}.

\begin{figure}[ht]
\centering
\includegraphics[width=0.7\columnwidth]{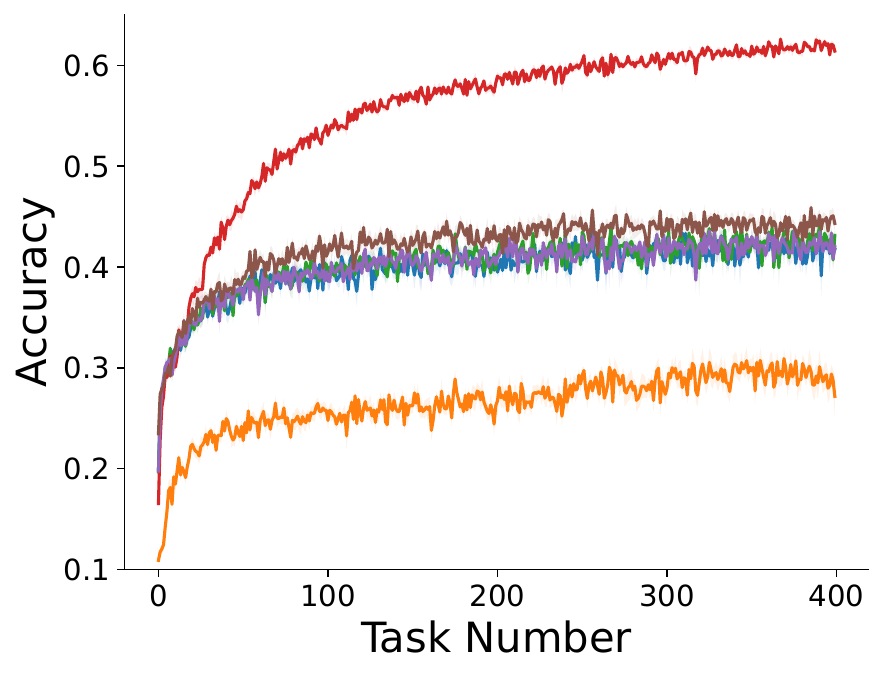}

\vspace{0.2cm}
\includegraphics[width=0.95\columnwidth]{figures/Legends/test_legend.pdf}
\vspace{-0.2cm}
\caption{Performance of UPGD, SGD, and Adam with weight decay, PGD and Shrink \& Perturb on Output-Permuted CIFAR-10. A global first-order utility is used with the two UPGDs.}
\label{labelpermuted-cifar10-upgd-global}
\vspace{-0.2cm}
\end{figure}

We compare SGDW, Shrink \& Perturb, AdamW, UPGD-W, and Non-protecting UPGD-W. All methods are tuned extensively (shown in Appendix \ref{appendix:details-output-permuted-cifar10}). Fig.\ \ref{labelpermuted-cifar10-upgd-global} shows the average online accuracy with the task number over $10$ independent runs. When more tasks are presented, the online accuracy of all methods improves. We suggest that the problem is easier than Output-permuted EMNIST since the number of classes in CIFAR-10 is less than EMNIST. A larger number of classes means that the probability of a class not changing after a permutation is significantly less. 

All methods maintain their plasticity. We notice that PGD performs worse than other methods emphasizing the role of weight decay in improving the performance in this task. In contrast to the previous two tasks, we see that AdamW performs better than SGDW. Similar to the previous tasks, UPGD-W has significantly better performance than other algorithms. This suggests that UPGD-W maintains plasticity and reduces forgetting such that at every task it can improve its representations. We found that UPGD-W performs the best when it has no weight decay in this task. In Appendix \ref{appendix:ablation-study}, we present an ablation study for the effect of weight decay and perturbation on performance in Shrink \& Perturb, UPGD-W, and Non-protecting UPGD-W.


\section{Related Works}
Neural network pruning requires a saliency or importance metric to choose which weights to prune. Typically, after training is complete, the network is pruned using measures such as the weight magnitude (e.g., Han et al.\ 2015, Park et al.\ 2020). Other metrics have been proposed using first-order information (e.g., Mozer \& Smolensky 1988, Hassibi \& Stork 1992, Molchanov et al.\ 2016), second-order information (e.g., LeCun et al.\ 1989, Dong et al.\ 2017), or both (e.g., Tresp et al.\ 1996, Molchanov et al.\ 2019). Molchanov et al.\ (2019) showed that the second-order Taylor's approximation of some true utility, which requires superlinear computation, closely matches the true utility. This result matches our derivations and conclusions. However, we use an efficient second-order Taylor's approximation using HesScale (Elsayed \& Mahmood 2022), making our method computationally inexpensive. Finally, our utility propagation method (derived in Appendix \ref{appendix:utility-propagation}) resembles, in principle, the approach pursued by Karnin (1990) and more recently by Yu et al.\ (2018), which proposes a method to propagate the importance scores from the final layer. However, our method uses a utility based on second-order approximation instead of reconstruction error or first-order approximation.

Evolutionary strategy (Rudolph 1997, Rechenberg 1973) is a derivative-free optimization class of methods that uses ideas inspired by natural evolution. Typically, a population of parameterized functions is created with different initializations, which can be improved using some score of success. Recent methods (e.g., Salimans et al.\ 2017, Such et al.\ 2017) use evolutionary strategies to improve a collection of parameterized policies based on the return coming from each one. The population is created by perturbing the current policy, then each policy is run for an episode to get a return. The new policy parameter update is the average of these perturbed policies' parameters weighted by their return corresponding returns. Evolutionary strategies, in comparison to UPS, search in the solution space using a population of solutions, making them infeasible beyond simulation. UPS, however, uses search in the weight or feature space, making UPS suitable for online learning. 

Feature-wise UPS can be seen as a generalization of the generate-and-test method (Mahmood \& Sutton 2013) for multi-layered networks with arbitrary objective functions. The generate-and-test method works only with networks with single-hidden layers in single-output regression problems and updates features in a conditional-selection fashion (e.g., replaces a feature if its utility is lower than other utilities). The feature utility is the trace of its outgoing weight magnitude. However, it has been shown that weight magnitude is not suitable for other problems, such as classification (Elsayed 2022). On the contrary, feature-wise UPS uses a better notion of utility that enables better search in the feature space and works with arbitrary network structures or objective functions.

Most continual learning methods use sequential tasks with known boundaries. Such a notion is unrealistic to be met in reality. There are a few task-free methods (e.g., Aljundi et al.\ 2019, Lee et al.\ 2020, He et al.\ 2019). However, the current task-free methods are either progressive to accommodate new experiences, require a replay buffer, or contain explicit task inference components. In comparison, UPGD is scalable in memory and computation, making it well-suited for lifelong agents that run for extended periods of time.


\section{Conclusion}
UPGD is a novel approach enabling learning agents to work for extended periods of time, making it well-suited for continual learning. We devised utility-informed learning rules that protect useful weights or features and perturb less useful ones. Such update rules help mitigate the issues encountered by modern representation learning methods in continual learning, namely catastrophic forgetting and decaying plasticity. We performed a series of experiments showing that UPGD helps in maintaining network plasticity and reusing previously learned features. Our results show that UPGD is well-suited for continual learning where the agent requires fast adaptation to an ever-changing world.

\begin{algorithm}[ht]
\caption{Feature-wise UPGD with global scaled utility}\label{alg:upgd-feature-global}
\begin{algorithmic}
\STATE {\bfseries Require:} Neural network $f$ with weights $\{\mW_1,...,\mW_L\}$ and a stream of data $\mathcal{D}$.
\STATE {\bfseries Require:} Step size $\alpha$ and decay rate $\beta \in [0,1)$.
\STATE {\bfseries Require:} Initialize weights $\{\mW_1,...,\mW_{L}\}$ randomly and masks $\{\vg_1,...,\vg_{L-1}\}$ to ones.
\STATE {\bfseries Require:} Initialize time step $t\leftarrow0$; $\eta\leftarrow\infty$;
\FOR{$l$ in $\{L,L-1,...,1\}$} 
\STATE $\vu_{l} \leftarrow \mathbf{0}$; $\hat{\vu}_{l} \leftarrow \mathbf{0}$
\ENDFOR
\FOR{$(\vx, \vy)$ in $\mathcal{D}$}
\STATE $t\leftarrow t+1$
\FOR{$l$ in $\{L-1,...,1\}$}
\STATE $\vf_{l},\vs_{l}\leftarrow$\texttt{GetMaskDerivatives}($f, \vx, \vy, l)$
\STATE $\vm_{l} \leftarrow \sfrac{1}{2}\vs_{l} - \vf_{l}$
\STATE $\vu_{l} \leftarrow \beta\vu_{l} + (1-\beta)\vm_l $
\STATE $\hat{\vu}_{l} \leftarrow \vu_{l}/(1-\beta^{t})$
\IF{$\eta < \texttt{max}(\hat{\vu}_{l})$}
\STATE $\eta \leftarrow \texttt{max}(\hat{\vu}_{l})$
\ENDIF
\ENDFOR
\FOR{$l$ in $\{L,L-1,...,1\}$}
\STATE $\mF_{l},\mS_{l}\leftarrow$\texttt{GetWeightDerivatives}($f, \vx, \vy, l)$
\IF{$l=L$}
\STATE $\mW_{l} \leftarrow \mW_{l} - \alpha\mF_{l}$
\ELSE
\STATE $\Bar{\vu}_{l} \leftarrow \phi{(\hat{\vu}_{l}/\eta)}$
\STATE Sample noise matrix $\boldsymbol{\xi}$
\STATE $\mW_{l} \leftarrow \mW_{l} - \alpha (\mF_{l} + \boldsymbol{\xi})\circ(1-\mathbf{1} \Bar{\vu}_{l}^{\top})$
\ENDIF
\ENDFOR
\ENDFOR
\end{algorithmic}
\end{algorithm}


\section{Broader Impact}
Temporally correlated data cause catastrophic forgetting, which makes reinforcement learning with function approximation challenging and sample inefficient (Fedus et al.\ 2020). The best-performing policy gradient methods (e.g., Mnih et al.\ 2015, Haarnoja et al.\ 2018) use a large replay buffer to reduce the correlation in successive transitions and make them seem independent and identically distributed to the learner. Since UPGD helps against catastrophic forgetting without the need for buffers, it can potentially improve the performance of most reinforcement learning algorithms and ameliorate their scalability.

UPGD can be integrated with step-size adaptation methods (e.g., Schraudolph 1999, Jacobsen et al.\ 2019), making them work well under non-stationarity. Moreover, UPGD can help better study and analyze step-size adaptation methods in isolation of catastrophic forgetting and decaying plasticity.


\section*{Acknowledgement}
We gratefully acknowledge funding from the Canada CIFAR AI Chairs program, the Reinforcement Learning and Artificial Intelligence (RLAI) laboratory, the Alberta Machine Intelligence Institute (Amii), and the Natural Sciences and Engineering Research Council (NSERC) of Canada. We would also like to thank Compute Canada for providing the computational resources needed.


\section*{References}
\hangin Abbas, Z., Zhao, R., Modayil, J., White, A., \& Machado, M.\ C.\ (2023). Loss of Plasticity in Continual Deep Reinforcement Learning. \textit{arXiv preprint arXiv:2303.07507}.\\
\hangin Ash, J., \& Adams, R. P. (2020). On warm-starting neural network training. \textit{Advances in Neural Information Processing Systems}, \textit{33}, 3884-3894.\\
\hangin Aljundi, R., Kelchtermans, K., \& Tuytelaars, T.\ (2019). Task-free continual learning. \textit{Conference on Computer Vision and Pattern Recognition} (pp. 11254-11263).\\
\hangin Aljundi, R., Babiloni, F., Elhoseiny, M., Rohrbach, M., \& Tuytelaars, T.\ (2018). Memory aware synapses: Learning what (not) to forget. \textit{Proceedings of the European Conference on Computer Vision} (pp. 139-154).\\
\hangin Caccia, M., Rodriguez, P., Ostapenko, O., Normandin, F., Lin, M., Page-Caccia, L., ... \& Charlin, L.\ (2020). Online fast adaptation and knowledge accumulation (OSAKA): a new approach to continual learning. \textit{Advances in Neural Information Processing Systems}, \textit{33}, 16532-16545.\\
\hangin Chaudhry, A., Rohrbach, M., Elhoseiny, M., Ajanthan, T., Dokania, P.\ K., Torr, P. H., \& Ranzato, M.\ A.\ (2019). On tiny episodic memories in continual learning. \textit{arXiv preprint arXiv:1902.10486}.\\
\hangin Cohen, G., Afshar, S., Tapson, J., \& Van Schaik, A.\ (2017). EMNIST: Extending MNIST to handwritten letters. \textit{International Joint Conference on Neural Networks} (pp. 2921-2926).\\
\hangin Dohare, S., Hernandez-Garcia, J.\ F., Rahman, P., Sutton, R.\ S., \& Mahmood, A.\ R.\ (2022). \textit{Maintaining Plasticity in Deep Continual Learning}. In preparation.\\
\hangin Dohare, S., Sutton, R.\ S., \& Mahmood, A.\ R. (2021). Continual backprop: Stochastic gradient descent with persistent randomness. \textit{arXiv preprint arXiv:2108.06325}.\\
\hangin Dong, X., Chen, S., \& Pan, S.\ (2017). Learning to prune deep neural networks via layer-wise optimal brain surgeon. \textit{Advances in Neural Information Processing Systems}, \textit{30}.\\
\hangin Elsayed, M., \& Mahmood, A.\ R.\ (2022). HesScale: Scalable Computation of Hessian Diagonals. \textit{arXiv preprint arXiv:2210.11639}.\\
\hangin Elsayed, M.\ (2022). \textit{Investigating Generate and Test for Online Representation Search with Softmax Outputs}. M.Sc.\ thesis, University of Alberta.\\
\hangin Fedus, W., Ghosh, D., Martin, J.\ D., Bellemare, M.\ G., Bengio, Y., \& Larochelle, H.\ (2020). On catastrophic interference in atari 2600 games. \textit{arXiv preprint arXiv:2002.12499}.\\
\hangin French, R.\ M.\ Using semi-distributed representations to overcome catastrophic forgetting in connectionist networks. (1991) \textit{Cognitive Science Society Conference} (pp. 173-178).\\
\hangin Haarnoja, T., Zhou, A., Abbeel, P., \& Levine, S.\ (2018). Soft actor-critic: Off-policy maximum entropy deep reinforcement learning with a stochastic actor. \textit{International Conference on Machine Learning}(pp. 1861-1870).\\
\hangin Hetherington, P.\ A., \& Seidenberg, M.\ S.\ (1989). Is there ‘catastrophic interference’ in connectionist networks? \textit{Conference of the Cognitive Science Society} (pp. 26-33).\\
\hangin Hebb, D.\ O.\ (1949). The organization of behavior: A neuropsychological theory. Wiley.\\
\hangin Han, S., Pool, J., Tran, J., \& Dally, W.\ (2015). Learning both weights and connections for efficient neural network. \textit{Advances in neural information processing systems}, 28.\\
\hangin He, X., Sygnowski, J., Galashov, A., Rusu, A.\ A., Teh, Y.\ W., \& Pascanu, R. (2019). Task agnostic continual learning via meta learning. \textit{arXiv preprint arXiv:1906.05201}.\\
\hangin He, K., Zhang, X., Ren, S., \& Sun, J.\ (2015). Delving deep into rectifiers: Surpassing human-level performance on imagenet classification. \textit{IEEE International Conference on Computer Vision} (pp. 1026-1034).\\
\hangin Hassibi, B., \& Stork, D.\ (1992). Second order derivatives for network pruning: Optimal brain surgeon. \textit{Advances in neural information processing systems}, \textit{5}.\\
\hangin Isele, D., \& Cosgun, A.\ (2018). Selective experience replay for lifelong learning. \textit{AAAI Conference on Artificial Intelligence} (pp. 3302-3309)\\
\hangin Jacobsen, A., Schlegel, M., Linke, C., Degris, T., White, A., \& White, M.\ (2019). Meta-descent for online, continual prediction. \textit{AAAI Conference on Artificial Intelligence} (pp. 3943-3950).\\
\hangin Karnin, E.\ D.\ (1990). A simple procedure for pruning back-propagation trained neural networks. \textit{IEEE transactions on neural networks}, \textit{1}(2), 239-242.\\
\hangin Kingma, D.\ P.\ \& Ba, J.\ (2015). Adam: A method for stochastic optimization. \textit{International Conference on Learning Representations}. \\
\hangin Kirkpatrick, J., Pascanu, R., Rabinowitz, N., Veness, J., Desjardins, G., Rusu, A.\ A., ... \& Hadsell, R.\ (2017). Overcoming catastrophic forgetting in neural networks. \textit{National Academy of Sciences}, \textit{114}(13), 3521-3526.\\
\hangin Konorski J.\ (1948). \textit{Conditioned Reflexes and Neuron Organization}. Cambridge University Press.\\
\hangin Krizhevsky, A.\ (2009) \textit{Learning Multiple Layers of Features from Tiny Images}.\ Ph.D. dissertation, University of Toronto.\\
\hangin Kunin, D., Sagastuy-Brena, J., Ganguli, S., Yamins, D.\ L., \& Tanaka, H. (2020). Neural mechanics: Symmetry and broken conservation laws in deep learning dynamics. \textit{arXiv preprint arXiv:2012.04728}.\\
\hangin Lee, S., Ha, J., Zhang, D., \& Kim, G. (2020). A neural dirichlet process mixture model for task-free continual learning. \textit{arXiv preprint arXiv:2001.00689}.\\
\hangin LeCun, Y., Denker, J., \& Solla, S. (1989). Optimal brain damage. \textit{Advances in neural information processing systems}, \textit{2}.\\
\hangin LeCun, Y., Bottou, L., Bengio, Y., \& Haffner, P.\ (1998). Gradient-based learning applied to document recognition. \textit{Proceedings of the IEEE}, \textit{86}(11), 2278-2324.\\
\hangin Liu, V., Kumaraswamy, R., Le, L., \& White, M.\ (2019). The utility of sparse representations for control in reinforcement learning. \textit{AAAI Conference on Artificial Intelligence} (pp. 4384-4391).\\
\hangin Loshchilov, I., \& Hutter, F.\ (2019). Decoupled weight decay regularization. \textit{International Conference on Learning Representations}.\\
\hangin Maas, A.\ L., Hannun, A.\ Y., \& Ng, A.\ Y.\ (2013). Rectifier nonlinearities improve neural network acoustic models. \textit{ICML Workshop on Deep Learning for Audio, Speech, and Language Processing}.\\
\hangin Mahmood, A.\ R., \& Sutton, R.\ S.\ (2013). Representation search through generate and test. \textit{AAAI Conference on Learning Rich Representations from Low-Level Sensors} (pp. 16-21).\\
\hangin Mozer, M. C., \& Smolensky, P.\ (1988). Skeletonization: A technique for trimming the fat from a network via relevance assessment. Advances in neural information processing systems, 1.\\
\hangin Melchior, J., \& Wiskott, L. (2019). Hebbian-descent. \textit{arXiv preprint arXiv:1905.10585}.\\
\hangin Molchanov, P., Mallya, A., Tyree, S., Frosio, I., \& Kautz, J.\ (2019). Importance estimation for neural network pruning. \textit{IEEE/CVF Conference on Computer Vision and Pattern Recognition} (pp. 11264-11272).\\
\hangin Molchanov, P., Tyree, S., Karras, T., Aila, T., \& Kautz, J.\ (2016). Pruning convolutional neural networks for resource efficient inference. \textit{arXiv preprint arXiv:1611.06440}.\\
\hangin Mnih, V., Kavukcuoglu, K., Silver, D., Rusu, A.\ A., Veness, J., Bellemare, M.\ G., ... \& Hassabis, D.\ (2015). Human-level control through deep reinforcement learning. \textit{Nature}, \textit{518}(7540), 529-533.\\
\hangin McCloskey, M., \& Cohen, N. J.\ (1989). Catastrophic interference in connectionist networks: The sequential learning problem. \textit{Psychology of Learning and Motivation}, \textit{24}, 109–165.\\
\hangin Nair, V., \& Hinton, G.\ E.\ (2010). Rectified linear units improve restricted boltzmann machines. \textit{International Conference on Machine Learning} (pp. 807-814).\\
\hangin Neelakantan, A., Vilnis, L., Le, Q.\ V., Sutskever, I., Kaiser, L., Kurach, K., \& Martens, J. (2015). Adding gradient noise improves learning for very deep networks. \textit{arXiv preprint arXiv:1511.06807}.\\
\hangin Lyle, C., Zheng, Z., Nikishin, E., Pires, B.\ A., Pascanu, R., \& Dabney, W.\ (2023). Understanding plasticity in neural networks. \textit{arXiv preprint arXiv:2303.01486}.\\
\hangin Orvieto, A., Kersting, H., Proske, F., Bach, F., \& Lucchi, A.\ (2022). Anticorrelated noise injection for improved generalization. \textit{arXiv preprint arXiv:2202.02831}.\\
\hangin Park, S., Lee, J., Mo, S., \& Shin, J. (2020). Lookahead: a far-sighted alternative of magnitude-based pruning. \textit{arXiv preprint arXiv:2002.04809}.\\
\hangin Tresp, V., Neuneier, R., \& Zimmermann, H.\ G.\ (1996). Early brain damage. \textit{Advances in Neural Information Processing Systems}, 9.\\
\hangin Rusu, A.\ A., Rabinowitz, N. C., Desjardins, G., Soyer, H., Kirkpatrick, J., Kavukcuoglu, K., ... \& Hadsell, R.\ (2016). Progressive neural networks. \textit{arXiv preprint arXiv:1606.04671}.\\
\hangin Rolnick, D., Ahuja, A., Schwarz, J., Lillicrap, T., \& Wayne, G.\ (2019). Experience replay for continual learning. \textit{Advances in Neural Information Processing Systems}, \textit{32}.\\
\hangin Rudolph, G.\ (1997). \textit{Convergence properties of evolutionary algorithms}. Verlag Dr. Kovač. \\
\hangin Rechenberg, I.\ (1973). Evolutionsstrategie. \textit{Optimierung technischer Systeme nach Prinzipien derbiologischen Evolution}.\\
\hangin Salimans, T., Ho, J., Chen, X., Sidor, S., \& Sutskever, I. (2017). Evolution strategies as a scalable alternative to reinforcement learning. \textit{arXiv preprint arXiv:1703.03864}.\\
\hangin Such, F. P., Madhavan, V., Conti, E., Lehman, J., Stanley, K. O., \& Clune, J. (2017). Deep neuroevolution: Genetic algorithms are a competitive alternative for training deep neural networks for reinforcement learning. \textit{arXiv preprint arXiv:1712.06567}.\\
\hangin Schraudolph, N.\ N.\ (1999) Local gain adaptation in stochastic gradient descent. \textit{International Conference on Artificial Neural Networks} (pp. 569-574).\\
\hangin Schwarz, J., Czarnecki, W., Luketina, J., Grabska-Barwinska, A., Teh, Y.\ W., Pascanu, R., \& Hadsell, R. (2018). Progress \& compress: A scalable framework for continual learning. \textit{International Conference on Machine Learning} (pp. 4528-4537).\\
\hangin Sutton, R.\ S.\ (1986). Two problems with backpropagation and other steepest-descent learning procedures for networks. \textit{Annual Conference of the Cognitive Science Society} (pp.\ 823-832).\\
\hangin Tanaka, H., Kunin, D., Yamins, D. L., \& Ganguli, S. (2020). Pruning neural networks without any data by iteratively conserving synaptic flow. \textit{Advances in Neural Information Processing Systems}, \textit{33}, 6377-6389.\\
\hangin Yu, R., Li, A., Chen, C.\ F., Lai, J. H., Morariu, V.\ I., Han, X., ... \& Davis, L. S.\ (2018). Nisp: Pruning networks using neuron importance score propagation. \textit{IEEE conference on computer vision and pattern recognition} (pp. 9194-9203).\\
\hangin Zhou, M., Liu, T., Li, Y., Lin, D., Zhou, E., \& Zhao, T.\ (2019). Toward understanding the importance of noise in training neural networks. \textit{International Conference on Machine Learning} (pp. 7594-7602).\\



\newpage
\appendix
\onecolumn
\section{Utility Propagation}
\label{appendix:utility-propagation}
The instantaneous utility measure can be used in a recursive formulation, allowing for backward propagation. We can get a recursive formula for the utility equation for connections in a neural network. This property is a result of \cref{thm:utility-propagation}.

\begin{theorem}
\label{thm:utility-propagation}
If the second-order off-diagonal terms in all layers in a neural network except for the last one are zero and all higher-order derivatives are zero, the true weight utility for the weight $ij$ at the layer $l$ can be propagated using the following recursive formulation:
\begin{align*}
U_{l,i,j}(Z) &\doteq f_{l,i,j} + s_{l,i,j} 
\end{align*}
where
\begin{align*}
f_{l,i,j} &\doteq - \sigma^{\prime}(a_{l,i}) \frac{h_{l-1, j}}{h_{l, j} } W_{l,i,j} \sum_{k=1}^{|\va_{l+1}|} \Bigg( f_{l+1,k,j}  \frac{W_{l+1, k,i}}{W_{l+1,k,j}}\Bigg),\\
s_{l,i,j} &\doteq  \frac{h^2_{l-1, j}  W^2_{l, i, j}}{h_{l, j} W_{l+1,i,j}} \sum_{k=1}^{|\va_{l+1}|} \Bigg( s_{l+1,k,j} \frac{W^2_{l+1,k,i} {\sigma^{\prime}}(a_{l,i})^2}{h_{l, j}W_{l+1, i, j}}  + f_{l+1, k,j} W_{l+1,k,i} \sigma^{\prime\prime}(a_{l,i})  \Bigg).
\end{align*}
\end{theorem}
\begin{proof} 

First, we start by writing the partial derivative of the loss with respect to each weight in terms of earlier partial derivatives in the next layers as follows:
\begin{align}
\frac{\partial \mathcal{L}}{\partial a_{l,i}} &=  \sum_{k=1}^{|\va_{l+1}|} \frac{\partial \mathcal{L}}{\partial a_{l+1,k}} \frac{\partial a_{l+1,k}}{\partial h_{l,i}} \frac{\partial h_{l,i}}{\partial a_{l,i}} = \sum_{k=1}^{|\va_{l+1}|} \frac{\partial \mathcal{L}}{\partial a_{l+1,k}} W_{l+1, k,i} \sigma^{\prime}(a_{l,i}),\label{equ:gradient-a}\\
\frac{\partial \mathcal{L}}{\partial W_{l,i,j}} &= \frac{\partial \mathcal{L}}{\partial a_{l,i}} \frac{\partial a_{l,i}}{\partial W_{l,i,j}} = \frac{\partial \mathcal{L}}{\partial a_{l,i}} h_{l-1,j}. \label{equ:gradient-W}
\end{align}
Next, we do the same with second-order partial derivatives as follows:
\begin{align}
\widehat{\frac{\partial^2 \mathcal{L}}{\partial {a^2_{l,i}}}} &\doteq \sum_{k=1}^{|\va_{l+1}|} \Bigg[ \widehat{\frac{\partial^2 \mathcal{L}}{\partial {a^2_{l+1,k}}}} W^2_{l+1,k,i} {\sigma^{\prime}}(a_{l,i})^2 + \frac{\partial \mathcal{L}}{\partial a_{l+1,k}} W_{l+1,k,i} \sigma^{\prime\prime}(a_{l,i}) \Bigg] \label{equ:hessian-approx-a},\\
\widehat{\frac{\partial^2 \mathcal{L}}{\partial {W^2_{l,i, j}}}} &\doteq \widehat{\frac{\partial^2 \mathcal{L}}{\partial {a^2_{l,i}}}}h^2_{l-1,j}. \label{equ:hessian-approx-W}
\end{align}

From here, we can derive the utility propagation formulation as the sum of two recursive quantities, $f_{l, ij}$ and $s_{l,ij}$. These two quantities represent the first and second-order terms in the Taylor approximation. Using Eq. \ref{equ:gradient-a}, Eq. \ref{equ:gradient-W}, Eq. \ref{equ:hessian-approx-a}, and Eq. \ref{equ:hessian-approx-W}, we can derive the recursive formulation as follows:
\begin{align}
U_{l,i,j}(Z) &\doteq -\frac{\partial \mathcal{L}(\mathcal{W}, Z)}{\partial W_{l,i,j}} W_{l,i,j} + \frac{1}{2}\frac{\partial^2 \mathcal{L}(\mathcal{W}, Z)}{\partial W^2_{l,ij}} W^2_{l,ij} \nonumber\\
&\approx -\frac{\partial \mathcal{L}(\mathcal{W}, Z)}{\partial W_{l,i,j}} W_{l,i,j} + \frac{1}{2}\widehat{\frac{\partial^2 \mathcal{L}(\mathcal{W}, Z)}{\partial W^2_{l,ij}}} W^2_{l,ij} \nonumber\\
&= -\sum_{k=1}^{|\va_{l+1}|} \frac{\partial \mathcal{L}}{\partial a_{l+1,k}} W_{l+1, k,i} \sigma^{\prime}(a_{l,i}) h_{l-1, j} W_{l,i,j} \nonumber \\
& \quad \;+ \frac{1}{2} \sum_{k=1}^{|\va_{l+1}|} \Bigg[ \widehat{\frac{\partial^2 \mathcal{L}}{\partial {a^2_{l+1,k}}}} W^2_{l+1,k,i} {\sigma^{\prime}}(a_{l,i})^2 + \frac{\partial \mathcal{L}}{\partial a_{l+1,k}} W_{l+1,k,i} \sigma^{\prime\prime}(a_{l,i}) \Bigg] h^2_{l-1, j} W^2_{l, i, j} \nonumber\\
&= f_{l, i, j} + s_{l, i,j}.
\end{align}

Now, we can write the first-order part $f_{l, i,j}$ and the second-order part $s_{l, i,j}$ as follows:
\begin{align}
f_{l, i,j} &= -\sum_{k=1}^{|\va_{l+1}|} \Bigg(\frac{\partial \mathcal{L}}{\partial a_{l+1,k}} W_{l+1, k,i} \sigma^{\prime}(a_{l,i}) h_{l-1, j} W_{l,i,j}\Bigg) \label{equ:first-matrix}\\
s_{l, i,j} &=\frac{1}{2} \sum_{k=1}^{|\va_{l+1}|} \Bigg(\widehat{\frac{\partial^2 \mathcal{L}}{\partial {a^2_{l+1,k}}}} W^2_{l+1,k,i} {\sigma^{\prime}}(a_{l,i})^2 + \frac{\partial \mathcal{L}}{\partial a_{l+1,k}} W_{l+1,k,i} \sigma^{\prime\prime}(a_{l,i}) \Bigg)h^2_{l-1, j}  W^2_{l, i, j} \label{equ:sec-matrix}
\end{align}
Using Eq. \ref{equ:first-matrix} and Eq. \ref{equ:sec-matrix}, we can write the recursive formulation for $f_{l,ij}$ and $s_{l,ij}$ as follows:
\begin{align}
f_{l, ij} &= -\sum_{k=1}^{|\va_{l+1}|} \Bigg(\sum_{m=1}^{|\va_{l+2}|} \Big(\frac{\partial \mathcal{L}}{\partial a_{l+2,m}} W_{l+2, m,k} \sigma^{\prime}(a_{l+1,k}) \Big) \frac{h_{l, j} W_{l+1,k,j}}{h_{l, j} W_{l+1,k,j}} W_{l+1, k,i} \sigma^{\prime}(a_{l,i}) h_{l-1, j} W_{l,i,j}\Bigg) \nonumber \\
&=  -\sum_{k=1}^{|\va_{l+1}|} \Bigg( f_{l+1,k,j}  \sigma^{\prime}(a_{l,i})  \frac{h_{l-1, j} W_{l,i,j}}{h_{l, j} W_{l+1,k,j}} W_{l+1, k,i}\Bigg) \nonumber \\
&=  - \sigma^{\prime}(a_{l,i}) \frac{h_{l-1, j}}{h_{l, j} } W_{l,i,j} \sum_{k=1}^{|\va_{l+1}|} \Bigg( f_{l+1,k,j}  \frac{W_{l+1, k,i}}{W_{l+1,k,j}}\Bigg)
\end{align}
\begin{align}
s_{l, i,j} &=\frac{1}{2} \sum_{k=1}^{|\va_{l+1}|} \Bigg(\widehat{\frac{\partial^2 \mathcal{L}}{\partial {a^2_{l+1,k}}}} W^2_{l+1,k,i} {\sigma^{\prime}}(a_{l,i})^2 + \frac{\partial \mathcal{L}}{\partial a_{l+1,k}} W_{l+1,k,i} \sigma^{\prime\prime}(a_{l,i}) \Bigg)h^2_{l-1, j}  W^2_{l, i, j} \nonumber \\
&=\frac{1}{2} \sum_{k=1}^{|\va_{l+1}|} \Bigg(\Big( \sum_{m=1}^{|\va_{l+2}|} \Big[ \widehat{\frac{\partial^2 \mathcal{L}}{\partial {a^2_{l+2,m}}}} W^2_{l+2,m,k} {\sigma^{\prime}}(a_{l+1,k})^2 + \frac{\partial \mathcal{L}}{\partial a_{l+2,m}} W_{l+2,m,k} \sigma^{\prime\prime}(a_{l,k}) \Big]  \Big) W^2_{l+1,k,i} {\sigma^{\prime}}(a_{l,i})^2  \nonumber\\
& \quad + \Big(\sum_{m=1}^{|\va_{l+2}|} \frac{\partial \mathcal{L}}{\partial a_{l+2,m}} W_{l+2, m,k} \sigma^{\prime}(a_{l,m})\Big) W_{l+1,k,i} \sigma^{\prime\prime}(a_{l,i}) \Bigg)h^2_{l-1, j}  W^2_{l, i, j} \nonumber \\
&=\sum_{k=1}^{|\va_{l+1}|} \Bigg(\frac{1}{2} \Big( \sum_{m=1}^{|\va_{l+2}|} \Big[ \widehat{\frac{\partial^2 \mathcal{L}}{\partial {a^2_{l+2,m}}}} W^2_{l+2,m,k} {\sigma^{\prime}}(a_{l+1,k})^2 + \frac{\partial \mathcal{L}}{\partial a_{l+2,m}} W_{l+2,m,k} \sigma^{\prime\prime}(a_{l,k}) \Big]  \Big) \frac{h^2_{l, j}  W^2_{l+1, i, j}}{h^2_{l, j}  W^2_{l+1, i, j}} W^2_{l+1,k,i} {\sigma^{\prime}}(a_{l,i})^2  \nonumber\\
& \quad + \Big(\sum_{m=1}^{|\va_{l+2}|} \frac{\partial \mathcal{L}}{\partial a_{l+2,m}} W_{l+2, m,k} \sigma^{\prime}(a_{l+1,m}) \frac{h_{l, j} W_{l+1,i,j}}{h_{l, j} W_{l+1,i,j}}\Big) W_{l+1,k,i} \sigma^{\prime\prime}(a_{l,i}) \Bigg) h^2_{l-1, j}  W^2_{l, i, j}  \nonumber \\
&= \frac{h^2_{l-1, j}  W^2_{l, i, j}}{h_{l, j} W_{l+1,i,j}} \sum_{k=1}^{|\va_{l+1}|} \Bigg( s_{l+1,k,j} \frac{W^2_{l+1,k,i} {\sigma^{\prime}}(a_{l,i})^2}{h_{l, j}W_{l+1, i, j}}  + f_{l+1, k,j} W_{l+1,k,i} \sigma^{\prime\prime}(a_{l,i})  \Bigg)
\end{align}

\end{proof}

\section{Feature Utility Approximation using Weight utility}
\label{appendix:feature-weight-connection}

Here, we derive the global utility of a feature $p$ at layer $l$, which is represented as the difference between the loss when the feature is removed and the regular loss. Instead of developing a direct feature utility, we derive the utility of a feature based on the utility of the input or output weights. The utility vector corresponding to $l$-th layer features is $\vu_l(Z)$, where $Z=(X, Y)$ denotes the sample, which is given by
\begin{align*}
u_{l,i}(Z) = \mathcal{L}(\mathcal{W}_{\neg[l,i]}, Z) - \mathcal{L}(\mathcal{W} , Z),
\end{align*}
where $\mathcal{W}$ is the set of matrices $\{\mW_1,...,\mW_L\}$, $\mathcal{L}(\mathcal{W}, Z)$ is the sample loss of a network parameterized by $\mathcal{W}$ on sample $Z$, and $\mathcal{W}_{\neg[l,i]}$ is the same as $\mathcal{W}$ except the weight $\mW_{l+1,i,j}$ is set to 0 for all values of $i$.

Note that the second-order Talyor's approximation of this utility depends on the off-diagonal elements of the Hessian matrix at each layer, since more than one weight is removed at once. For our analysis, we drop these elements and derive our feature utility. We expand the difference around the current input weights of the feature $i$ at layer $l$ and evaluate it by setting the weight to zero. The quadratic approximation of $u_{l,i}(Z)$ can be written as
\begin{align}
u_{l,j}(Z) &= \mathcal{L}(\mathcal{W}_{\neg[l,i]}, Z) - \mathcal{L}(\mathcal{W} , Z) \nonumber \\
&= \sum_{i=1}^{|\va_{l+1}|}\Big( -\frac{\partial \mathcal{L}}{\partial W_{l,i,j}} W_{l,i,j} + \frac{1}{2}\frac{\partial^2 \mathcal{L}}{\partial W^2_{l,ij}} W^2_{l+1,i, j} + 2 \sum_{j\neq i}\frac{\partial^2 \mathcal{L}}{\partial W_{l,i,j}\partial W_{l,i,k}} W_{l,i,j}W_{l,i,k} \Big)  \nonumber \\
&\approx \sum_{i=1}^{|\va_{l+1}|} \Big( -\frac{\partial \mathcal{L}}{\partial W_{l,ij}} W_{l,ij} + \frac{1}{2}\frac{\partial^2 \mathcal{L}}{\partial W^2_{l,ij}} W^2_{l,ij} \Big) \nonumber \\
&= \sum_{i=1}^{|\va_{l+1}|} U_{l,i,j}.
\label{equ:approx-feature-utility}
\end{align}

Alternatively, we can derive the utility of the $p$ at layer $l$ by dropping the input weights when the activation function passes through the origin (zero input leads to zero output). This gives rise to the property of \textit{conservation of utility} shown by \cref{thm:conservation-utility}.

\begin{theorem}
\label{thm:conservation-utility}
If the second-order off-diagonals in all layers in a neural network except for the last one are zero, all higher-order derivatives are zero, and an origin-passing activation function is used, the sum of output-weight utilities to a feature equals the sum of input-weight utilities to the same feature.
\begin{align*}
\sum_{i=1}^{|\va_{l+1}|} U_{l,i,j} &= \sum_{i=1}^{|\va_{l}|} U_{l-l,j,i}.
\end{align*}
\end{theorem}

\begin{proof} 
From Eq.\ \ref{equ:approx-feature-utility}, we can write the utility of the feature $p$ at layer $l$ by dropping the input weights when the activation function passes through the origin (zero input leads to zero output) as follows:
\begin{align}
u_{l,j}(Z) &\approx \sum_{i=1}^{|\va_{l}|} U_{l-l,j,i}(Z). \nonumber
\end{align}
In addition, we can write the utility of the feature $p$ at layer $l$ by dropping the output weights as follows:
\begin{align}
u_{l,j}(Z) &\approx \sum_{i=1}^{|\va_{l}|} U_{l,i,j}(Z). \nonumber
\end{align}

Therefore, we can write the following equality:
\begin{align}
 \sum_{i=1}^{|\va_{l+1}|} U_{l,i,j} &= \sum_{i=1}^{|\va_{l}|} U_{l-l,j,i}. \nonumber
\end{align}
\end{proof}

This theorem shows the property of the conservation of instantaneous utility. The sum of utilities of the outgoing weights to a feature equals to the sum of utilities of the incoming weights to the same feature when origin-passing activation functions are used and the off-diagonal elements are dropped.

This conservation law resembles the one by Tanaka et al.\ (2020). Such a conservation law comes from a corresponding symmetry, namely the rescale symmetry (Kunin et al.\ 2020)

\clearpage
\section{Utility-based Perturbed Gradient Descent Algorithms with Local Utilities}
\begin{minipage}{0.5\textwidth}
\begin{algorithm}[H]
\caption{Weight-wise UPGD with local scaled utility}\label{alg:upgd-weight-layerwise}
\begin{algorithmic}
\STATE {\bfseries Require:} Neural network $f$ with weights $\{\mW_1,...,\mW_L\}$ and a stream of data $\mathcal{D}$
\STATE {\bfseries Require:} Step size $\alpha$ and decay rate $\beta \in [0,1)$
\STATE {\bfseries Require:} Initialize $\{\mW_1,...,\mW_L\}$
\STATE Initialize time step $t\leftarrow0$
\FOR{$l$ in $\{L,L-1,...,1\}$} 
\STATE $\mU_{l} \leftarrow \mathbf{0}$
\ENDFOR
\FOR{$(\vx, y)$ in $\mathcal{D}$}
\STATE $t\leftarrow t+1$
\FOR{$l$ in $\{L,L-1,...,1\}$}
\STATE $\mF_{l},\mS_{l}\leftarrow$\texttt{GetWeightDerivatives}($f, \vx, \vy, l)$
\STATE $\mM_{l} \leftarrow \sfrac{1}{2}\mS_{l}\circ\mW^2_{l} - \mF_{l}\circ\mW $
\STATE $\mU_{l} \leftarrow \beta\mU_{l} + (1-\beta)\mM_l $
\STATE $\hat{\mU}_{l} \leftarrow \mU_{l}/(1-\beta^{t})$
\STATE Sample noise matrix $\boldsymbol{\xi}$.
\STATE $\Bar{\mU}_{l} \leftarrow \phi{(\hat{\mU_{l}}}\mD)$ \quad \texttt{// where $D_{ii}=1/|\mU_{i,:,l}|$}
\STATE \texttt{\quad \quad \quad \quad \quad \quad \quad// and $D_{ij}=0, \forall i\neq j$}
\STATE $\mW_{l} \leftarrow \mW_{l} - \alpha (\mF_{l} + \boldsymbol{\xi})\circ(1-\Bar{\mU}_{l})$
\ENDFOR
\ENDFOR
\STATE 
\STATE 
\STATE 
\STATE 
\STATE 
\end{algorithmic}
\end{algorithm}
\end{minipage}
\hfill
\begin{minipage}{0.5\textwidth}
\begin{algorithm}[H]
\caption{Feature-wise UPGD with local scaled utility}\label{alg:upgd-feature-layerwise}
\begin{algorithmic}
\STATE {\bfseries Require:} Neural network $f$ with weights $\{\mW_1,...,\mW_L\}$ and a stream of data $\mathcal{D}$
\STATE {\bfseries Require:} Step size $\alpha$ and decay rate $\beta \in [0,1)$
\STATE {\bfseries Require:} Initialize weights $\{\mW_1,...,\mW_{L}\}$ randomly and masks $\{\vg_1,...,\vg_{L-1}\}$ to ones
\STATE {\bfseries Require:} Initialize time step $t\leftarrow 0$
\FOR{$l$ in $\{L,L-1,...,1\}$} 
\STATE $\vu_{l} \leftarrow \mathbf{0}$
\ENDFOR
\FOR{$(\vx, y)$ in $\mathcal{D}$}
\STATE $t\leftarrow t+1$
\FOR{$l$ in $\{L,L-1,...,1\}$}
\STATE $\mF_{l},\mS_{l}\leftarrow$\texttt{GetWeightDerivatives}($f, \vx, \vy, l)$
\IF{$l=L$}
\STATE $\mW_{l} \leftarrow \mW_{l} - \alpha\mF_{l}$
\ELSE
\STATE $\vf_{l},\vs_{l}\leftarrow$\texttt{GetMaskDerivatives}($f, \vx, \vy, l)$
\STATE $\vm_{l} \leftarrow \sfrac{1}{2}\vs_{l} - \vf_{l}$
\STATE $\vu_{l} \leftarrow \beta\vu_{l} + (1-\beta)\vm_l $
\STATE $\hat{\vu}_{l} \leftarrow \vu_{l}/(1-\beta^{t})$
\STATE $\Bar{\vu}_{l} \leftarrow \phi{(\hat{\vu}_{l}/|\hat{\vu}_{l}|)}$
\STATE Sample noise matrix $\boldsymbol{\xi}$.
\STATE $\mW_{l} \leftarrow \mW_{l} - \alpha (\mF_{l} + \boldsymbol{\xi})\circ(1-\mathbf{1} \Bar{\vu}_{l}^{\top})$
\ENDIF
\ENDFOR
\ENDFOR
\end{algorithmic}
\end{algorithm}
\end{minipage}


\section{Experimental Details}
\label{appendix:experimental-details}
\subsection{Toy Problem}
\label{appendix:details-toy}
The learners use a network of two hidden layers containing $300$ and $150$ units with identity activations, respectively. A thorough hyperparameter search is conducted. The search space for step size is $\{10^{-i}\}_{i=1}^{4}$, for the utility trace decay $b_u$ is $\{0.0, 0.9, 0.99, 0.999\}$, for the noise standard deviation is $\{10^{-i}\}_{i=0}^{4}$, and for the weight decay factor $\lambda$ is $\{0.1, 0.01, 0.001\}$. The utility traces are computed using exponential moving averages given by $\tilde{U}_t = \beta_u \tilde{U}_{t-1} + (1-\beta_u) U_{t}$, where $\tilde{U}_t$ is the utility trace at time $t$ and $U_{t}$ is the instantaneous utility at time $t$. Each learner is trained for $1$ million time steps. The results are averaged over $20$ independent runs.

\subsection{Stationary MNIST}
\label{appendix:details-stationary-mnist}
The learners use a network of two hidden layers containing $300$ and $150$ units with ReLU activations, respectively. A thorough hyperparameter search is conducted. For the search experiment using UPS, search space for the standard deviation of noise $\sigma$ is $\{10^{-i}\}_{i=1}^{5}$, the step size $\alpha$ is fixed for a value of  $1.0$, and the utility trace decay $b_u$ is $\{0.0, 0.9, 0.99, 0.999, 0.9999\}$. In the experiments using UPGD, the search space for step size $\alpha$ is $\{10^{-i}\}_{i=1}^{4}$, the utility trace decay $b_u$ is $\{0.0, 0.9, 0.99, 0.999\}$, the standard deviation of the noise is $\{10^{-i}\}_{i=0}^{4}$, and the weight decay factor $\lambda$ is $\{0.1, 0.01, 0.001\}$ for Shrink \& Perturb. The utility traces are computed using exponential moving averages given by $\tilde{U}_t = \beta_u \tilde{U}_{t-1} + (1-\beta_u) U_{t}$, where $\tilde{U}_t$ is the utility trace at time $t$ and $U_{t}$ is the instantaneous utility at time $t$. Each learner is trained for $1$ million time steps. The results are averaged over $20$ independent runs.

\subsection{Input-permuted MNSIT and Output-permuted EMNIST}
\label{appendix:details-input-permuted-mnist}
\label{appendix:details-output-permuted-emnist}
The learners use a network of two hidden layers containing $300$ and $150$ units with ReLU activations, respectively. A thorough hyperparameter search is conducted. The search space for step size $\alpha$ is $\{0.01, 0.001, 0.0001\}$, for the utility trace decay $b_u$ is $\{0.999, 0.9999\}$, for the standard deviation of the noise is $\{0.01, 0.1, 1.0\}$, and for the weight decay factor $\lambda$ is $\{0.0, 0.1, 0.01, 0.001, 0.0001\}$. The utility traces are computed using exponential moving averages given by $\tilde{U}_t = \beta_u \tilde{U}_{t-1} + (1-\beta_u) U_{t}$, where $\tilde{U}_t$ is the utility trace at time $t$ and $U_{t}$ is the instantaneous utility at time $t$. For AdamW, we used a search space for $\beta_1$ of $\{0.0, 0.9\}$, for $\beta_2$ of $\{0.99, 0.999, 0.9999\}$, and for $\epsilon$ of $\{10^{-4}, 10^{-8}, 10^{-16}\}$. Each learner is trained for $1$ million time steps. The results are averaged over $20$ independent runs.

\subsection{Output-permuted CIFAR-10}
\label{appendix:details-output-permuted-cifar10}
The learners use a network with two convolutional layers with max-pooling followed by two fully connected layers with ReLU activations. The first convolutional layer uses a kernel of $5$ and outputs $6$ filters, whereas the second convolutional layer uses a kernel of $5$ and outputs $16$ filters. The max-pooling layers use a kernel of $2$. The data is flattened after the second max-pooling layer and fed to a fully connected network with two hidden layers containing $120$ and $84$ units, respectively. A thorough hyperparameter search is conducted. The search space for step size $\alpha$ is $\{0.01, 0.001, 0.0001\}$, for the utility trace decay $b_u$ is $\{0.999, 0.9999\}$, for the standard deviation of the noise is $\{0.01, 0.1, 1.0\}$, and for the weight decay factor $\lambda$ is $\{0.0, 0.1, 0.01, 0.001, 0.0001\}$. The utility traces are computed using exponential moving averages given by $\tilde{U}_t = \beta_u \tilde{U}_{t-1} + (1-\beta_u) U_{t}$, where $\tilde{U}_t$ is the utility trace at time $t$ and $U_{t}$ is the instantaneous utility at time $t$. For AdamW, we used a search space for $\beta_1$ of $\{0.0, 0.9\}$, for $\beta_2$ of $\{0.99, 0.999, 0.9999\}$, and for $\epsilon$ of $\{10^{-4}, 10^{-8}, 10^{-16}\}$. Each learner is trained for $1$ million time steps. The results are averaged over $10$ independent runs.

\section{The HesScale Algorithm (Elsayed \& Mahmood 2022)}
\label{appendix:hesscale}
\begin{algorithm}
\caption{The HesScale Algorithm in Classification (Elsayed \& Mahmood 2022)}\label{alg:hesscale}
\begin{algorithmic}
\STATE {\bfseries Require:} Neural network $f$ and a layer number $l$
\STATE {\bfseries Require:} $\widehat{\frac{\partial \mathcal{L}}{\partial a_{l+1,i}}}$ and $\widehat{\frac{\partial^2 \mathcal{L}}{\partial a_{l+1,i,j}^2}}$, unless $l=L$
\STATE {\bfseries Require:} Input-output pair $(\vx, y)$.
\STATE Compute preference vector $\va_L \leftarrow f(\vx)$ and target one-hot-encoded vector $\vp\leftarrow\texttt{onehot}(y)$.
\STATE Compute the predicted probability vector $\vq\leftarrow\vsigma(\va_L)$.
\STATE Compute the error  $\mathcal{L}(\vp, \vq)$.
\IF{$l=L$} 
    \STATE Compute $\frac{\partial \mathcal{L}}{\partial \va_{L}} \leftarrow \vq -\vp$
    \STATE Compute $\frac{\partial \mathcal{L}}{\partial \mW_{L}}$ using Eq.\ \ref{equ:gradient-W} 
    \STATE $\widehat{\frac{\partial^2 \mathcal{L}}{\partial \va_{L}^2}} \leftarrow \vq-\vq\circ\vq$
    \STATE Compute $\widehat{\frac{\partial^2 \mathcal{L}}{\partial \mW_{L}^2}}$ using Eq.\ \ref{equ:hessian-approx-W}
\ELSIF{$l\neq L$}
    \STATE Compute $\frac{\partial \mathcal{L}}{\partial \va_{l}}$ and $\sfrac{\partial \mathcal{L}}{\partial \mW_{l}}$ using Eq. \ref{equ:gradient-a} and Eq. \ref{equ:gradient-W}
    \STATE Compute $\widehat{\frac{\partial^2 \mathcal{L}}{\partial \va_{l}^2}}$ and $\widehat{\frac{\partial^2 \mathcal{L}}{\partial \mW_{l}^2}}$ using Eq.\ \ref{equ:hessian-approx-a} and Eq.\ \ref{equ:hessian-approx-W}
\ENDIF
\STATE {\bfseries Return} $\frac{\partial \mathcal{L}}{\partial \mW_{l}}$, $\widehat{\frac{\partial^2 \mathcal{L}}{\partial \mW_{l}^2}}$, $\frac{\partial \mathcal{L}}{\partial \va_{l}}$, and $\widehat{\frac{\partial^2 \mathcal{L}}{\partial \va_{l}^2}}$
\end{algorithmic}
\end{algorithm}

\section{Additional Results}

\subsection{Utility-based Perturbed Search on Stationary MNIST using uncorrelated noise}
\label{appendix:additional-mnist-ups}

We repeat the stationary MNIST experiment on Utility-based Perturbed Search (UPS) but using uncorrelated noise. We report the results in Fig.\ \ref{fig:ups-uncorrelated} with uncorrelated noise. The UPS is unable to learn using uncorrelated noise in the weight space, but it is still able to learn in the feature space. This may suggest that learning in the weight space is harder since it is larger than the feature space, and the anti-correlated noise is more effective in larger spaces.

\begin{figure}[ht]
\centering
\subfigure[Weight-wise]{
    \includegraphics[width=0.40\textwidth]{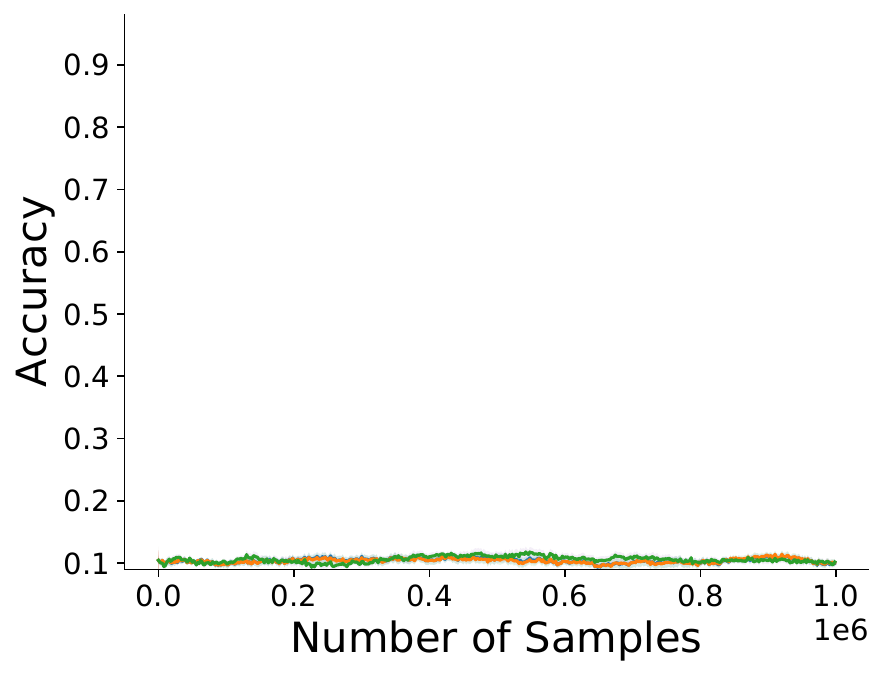}
}
\subfigure[Feature-wise]{
    \includegraphics[width=0.40\textwidth]{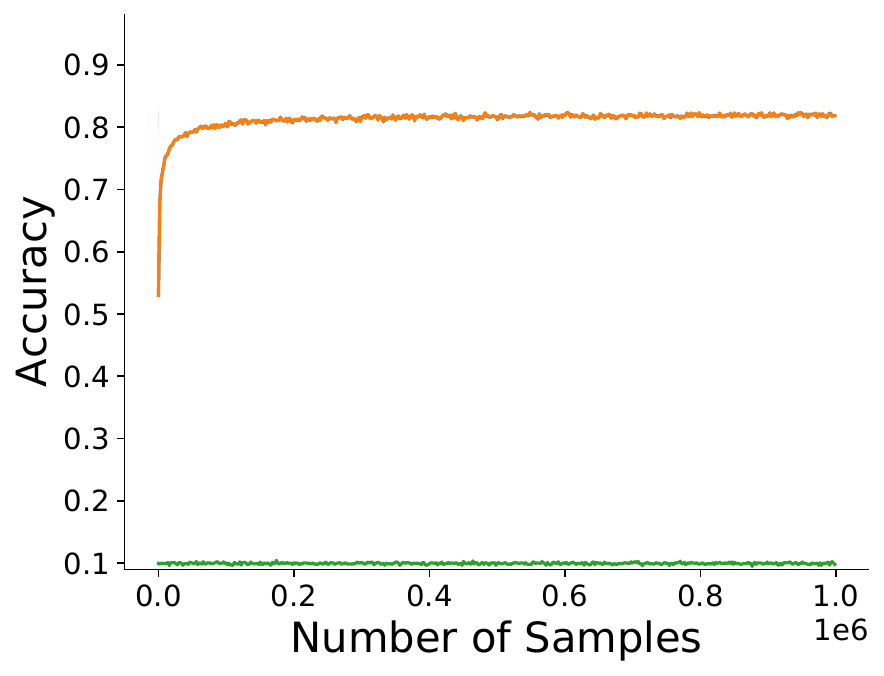}
}

\vspace{0.2cm}
\includegraphics[width=0.55\columnwidth]{figures/Legends/utility-mnist.pdf}
\vspace{-0.2cm}
\caption{Utility-based Perturbed Search with Uncorrelated Noise}
\label{fig:ups-uncorrelated}
\end{figure}


\subsection{UPGD on Stationary MNIST}
\label{appendix:additional-mnist-upgd}
We repeat the stationary MNIST experiment but with utility-based gradient descent. A desirable property of continual learning systems is that they should not asymptotically impose any extra performance reduction, which can be studied in a stationary task such as MNIST. We report the results in Fig.\ \ref{fig:mnist-upgd-all} with uncorrelated noise. We notice that UPGD improves performance over SGD since it uses search in addition to the gradient information.

\begin{figure}[H]
\centering
\subfigure[Weight-wise with Global Utility]{
    \includegraphics[width=0.32\textwidth]{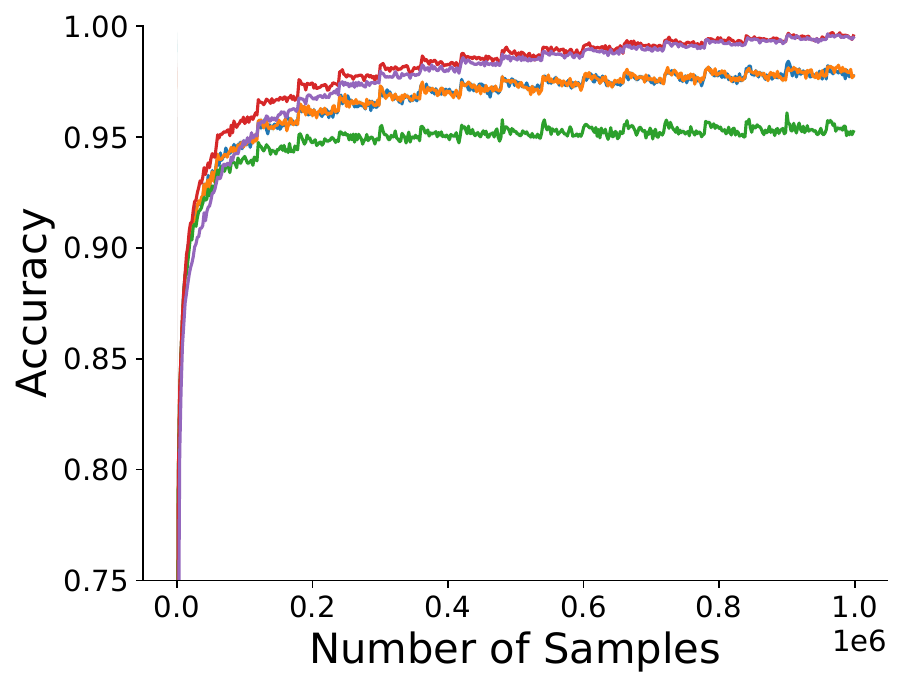}
}
\subfigure[Feature-wise with Global Utility]{
    \includegraphics[width=0.32\textwidth]{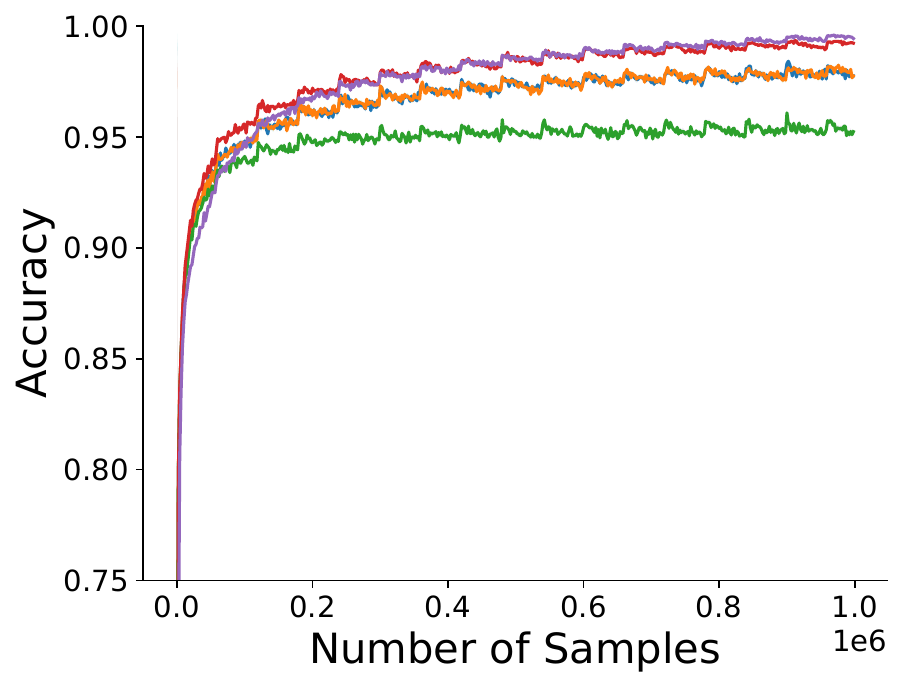}
}
\hfill

\subfigure[Weight-wise with Local Utility]{
    \includegraphics[width=0.32\textwidth]{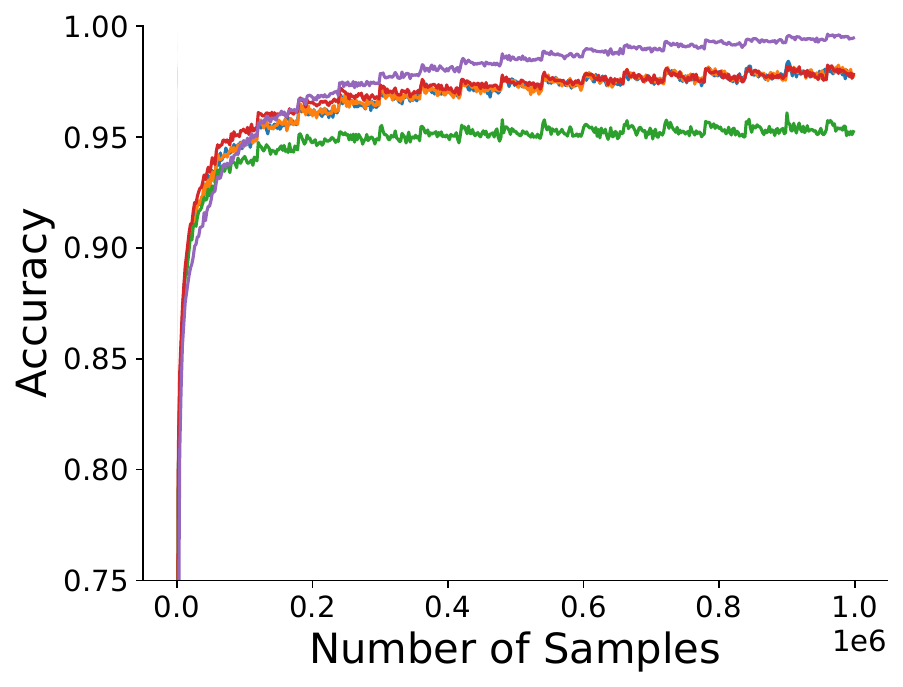}
}
\subfigure[Feature-wise with Local Utility]{
    \includegraphics[width=0.32\textwidth]{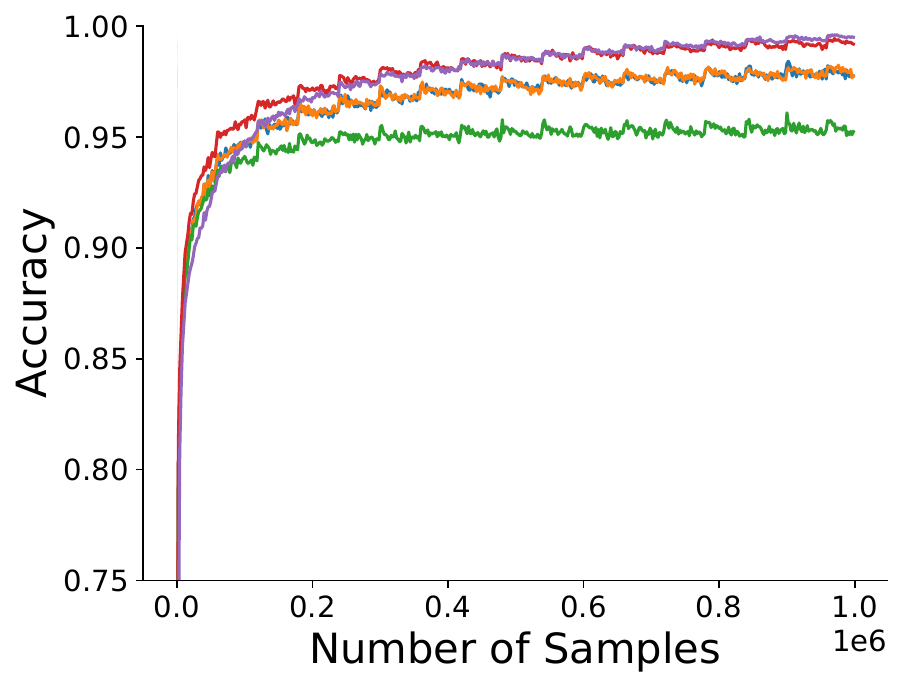}
}

\vspace{0.2cm}
\includegraphics[width=0.6\columnwidth]{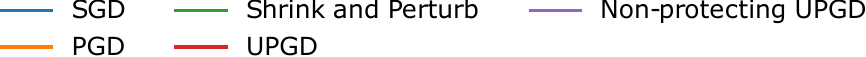}
\vspace{-0.2cm}
\caption{Utility-based Perturbed Gradient Descent with first-order approximated utilities using uncorrelated noise on stationary MNIST.}
\label{fig:mnist-upgd-all}
\end{figure}

\subsection{The Quality of the Approximated Utility}
\label{appendix:utility-additional}
Here, we provide more results for the quality of the approximated utility. We use the Spearman correlation to calculate the similarity between the ordering of the approximated utility against the true utility using different activation functions. Fig.\ \ref{fig:utility-additional} shows the Spearman correlation at every time step for the global and local cases.

\begin{figure}[ht]
\centering
\subfigure[Global weight utility (ReLU)]{
    \includegraphics[width=0.23\textwidth]{figures/ApproxUtility/weight/relu/global_correlations.pdf}
}
\subfigure[Local weight utility  (ReLU)]{
    \includegraphics[width=0.23\textwidth]{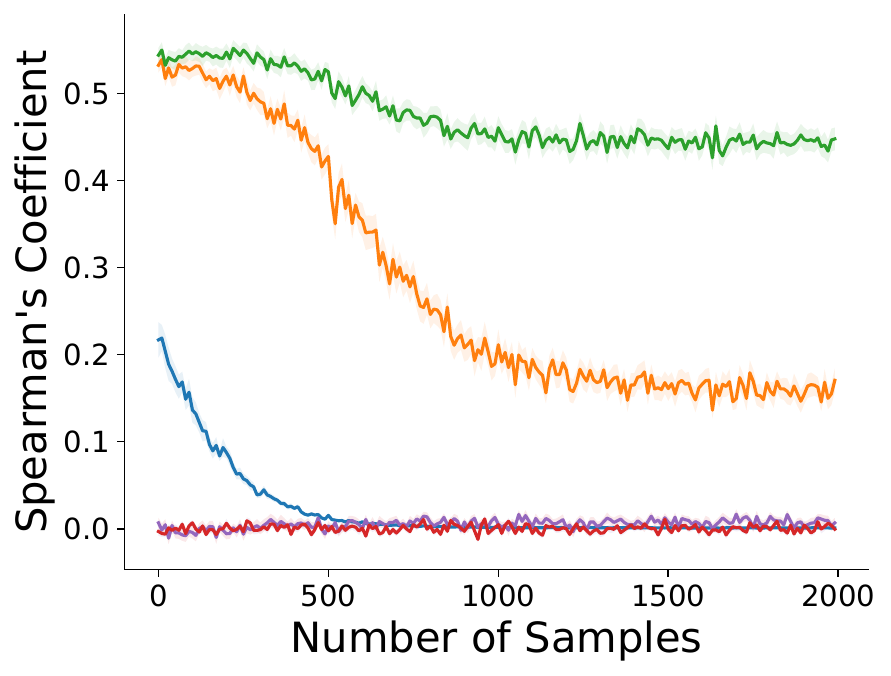}
}
\subfigure[Global feature utility  (ReLU)]{
    \includegraphics[width=0.23\textwidth]{figures/ApproxUtility/feature/relu/global_correlations.pdf}
}
\subfigure[Local feature utility  (ReLU)]{
    \includegraphics[width=0.23\textwidth]{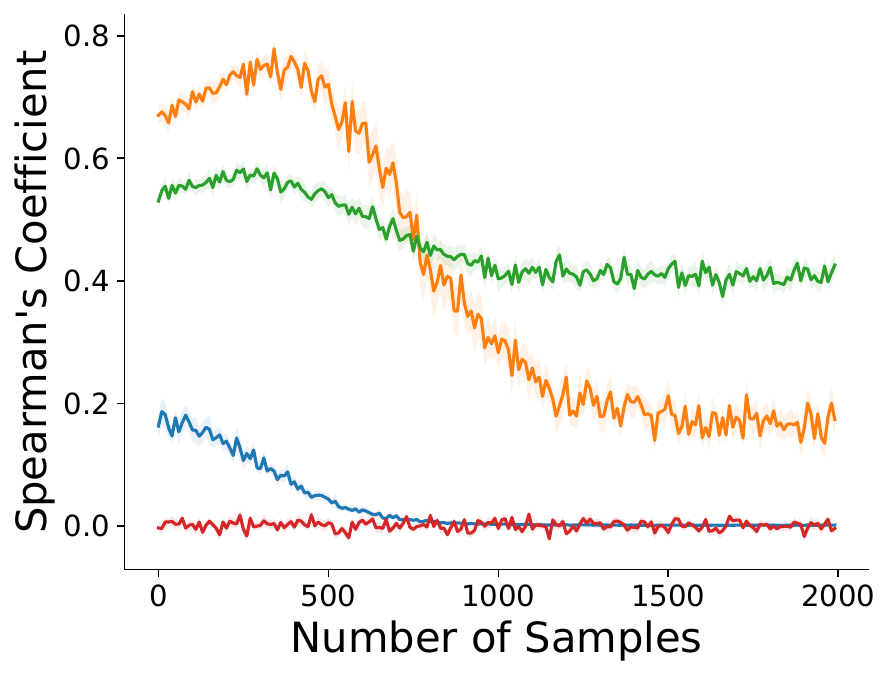}
}

\hfill

\subfigure[Global weight utility (LeakyReLU)]{
    \includegraphics[width=0.23\textwidth]{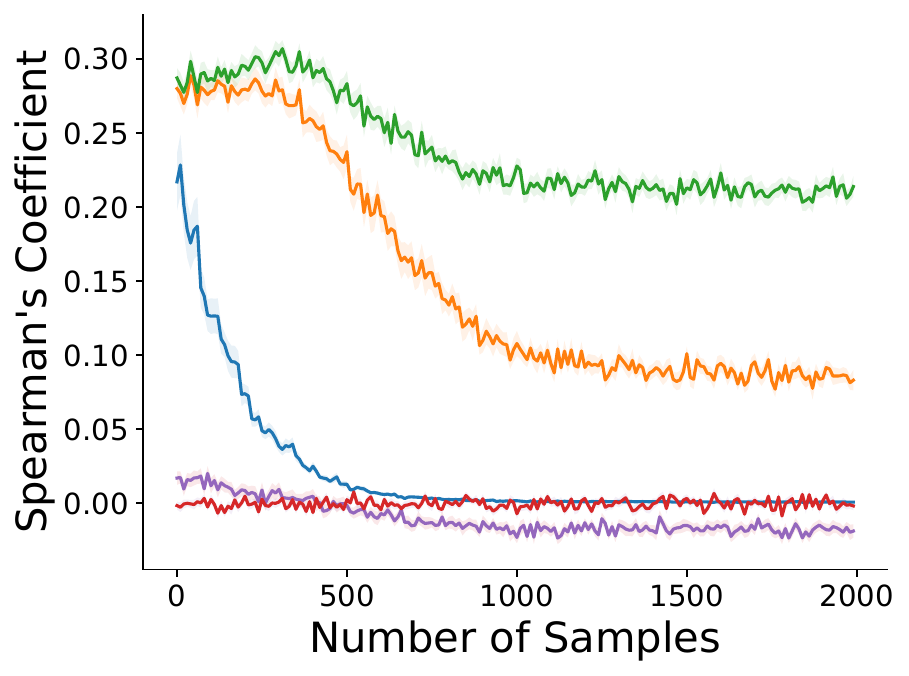}
}
\subfigure[Local weight utility (LeakyReLU)]{
    \includegraphics[width=0.23\textwidth]{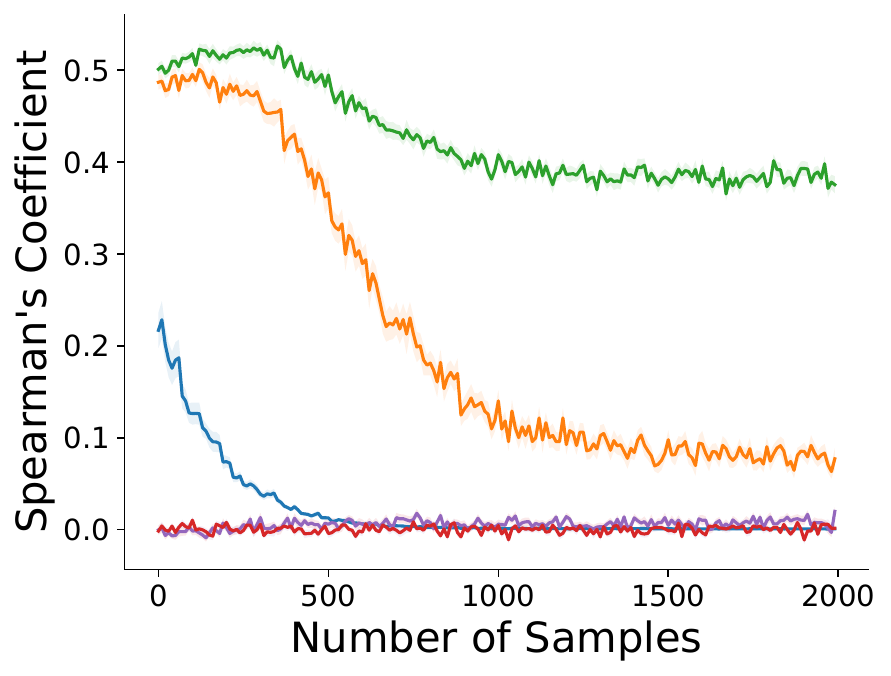}
}
\subfigure[Global feature utility (LeakyReLU)]{
    \includegraphics[width=0.23\textwidth]{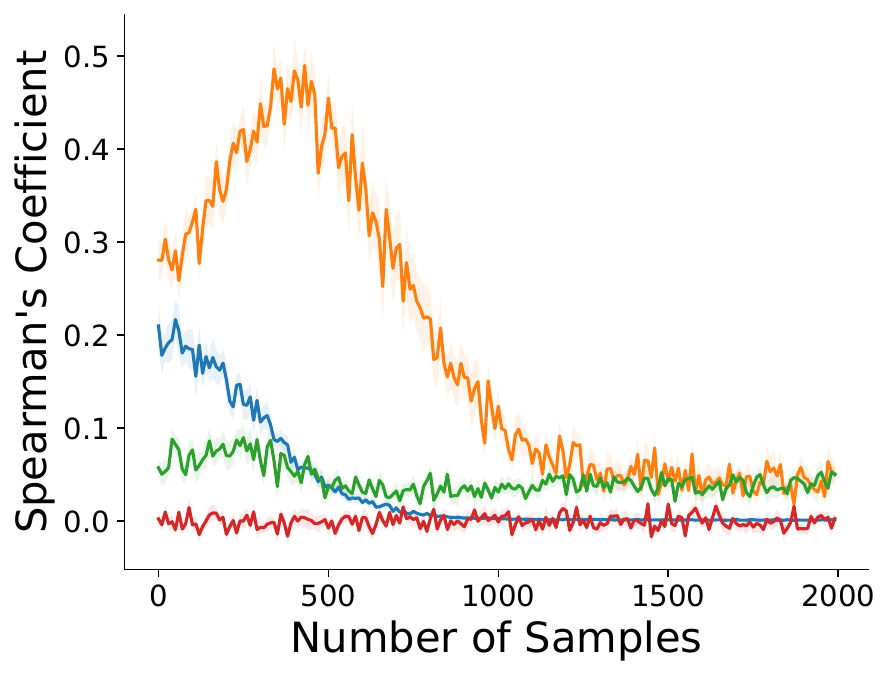}
}
\subfigure[Local feature utility (LeakyReLU)]{
    \includegraphics[width=0.23\textwidth]{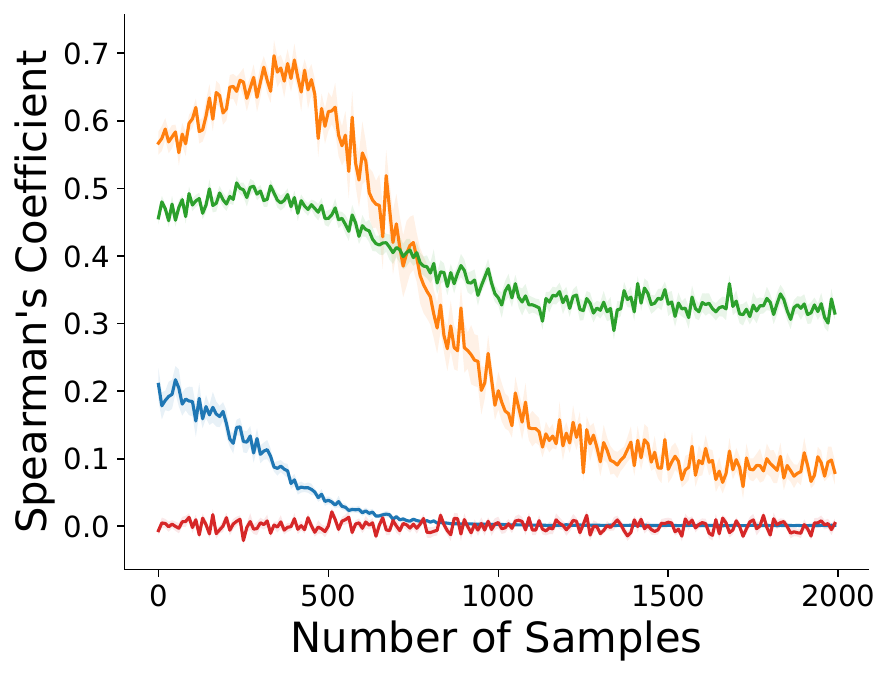}
}

\hfill

\subfigure[Global weight utility (Tanh)]{
    \includegraphics[width=0.23\textwidth]{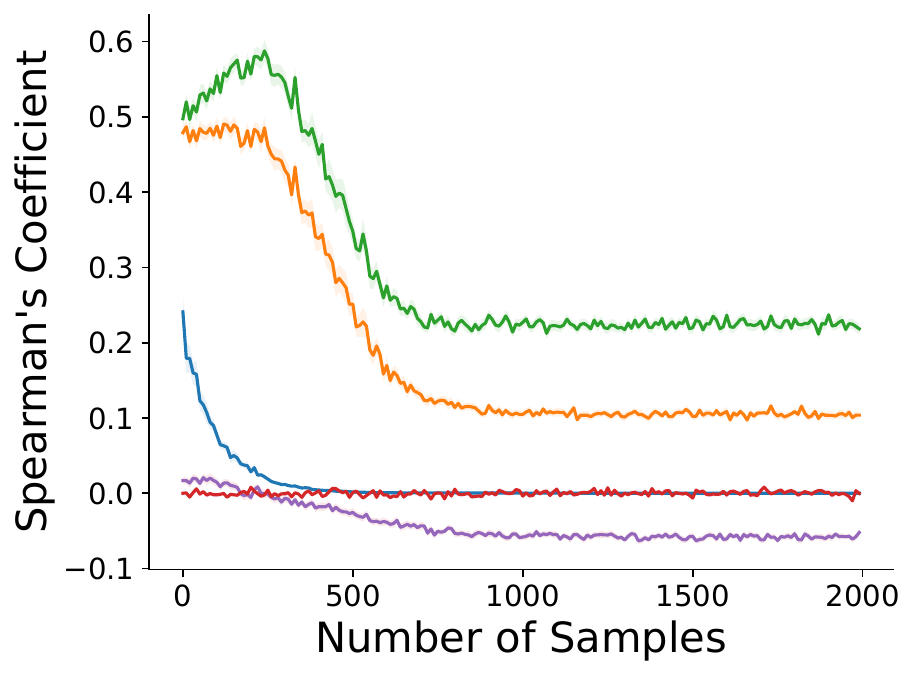}
}
\subfigure[Local weight utility (Tanh)]{
    \includegraphics[width=0.23\textwidth]{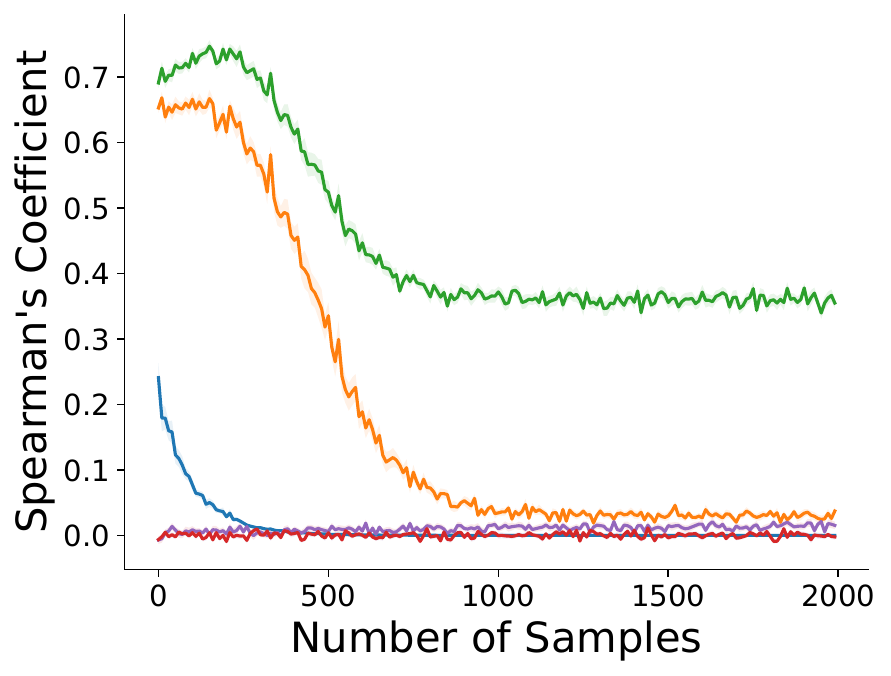}
}
\subfigure[Global feature utility (Tanh)]{
    \includegraphics[width=0.23\textwidth]{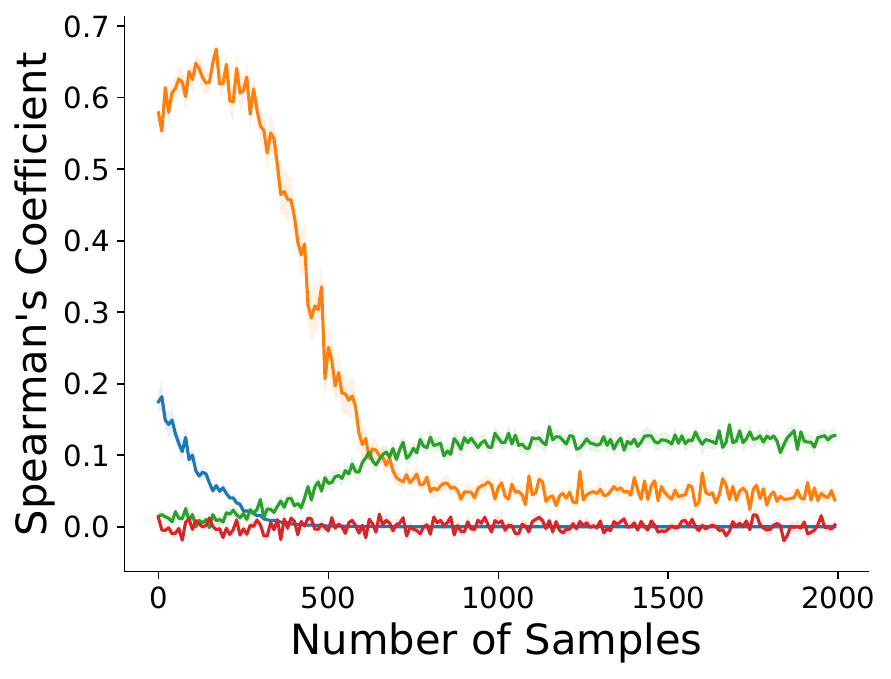}
}
\subfigure[Local feature utility (Tanh)]{
    \includegraphics[width=0.23\textwidth]{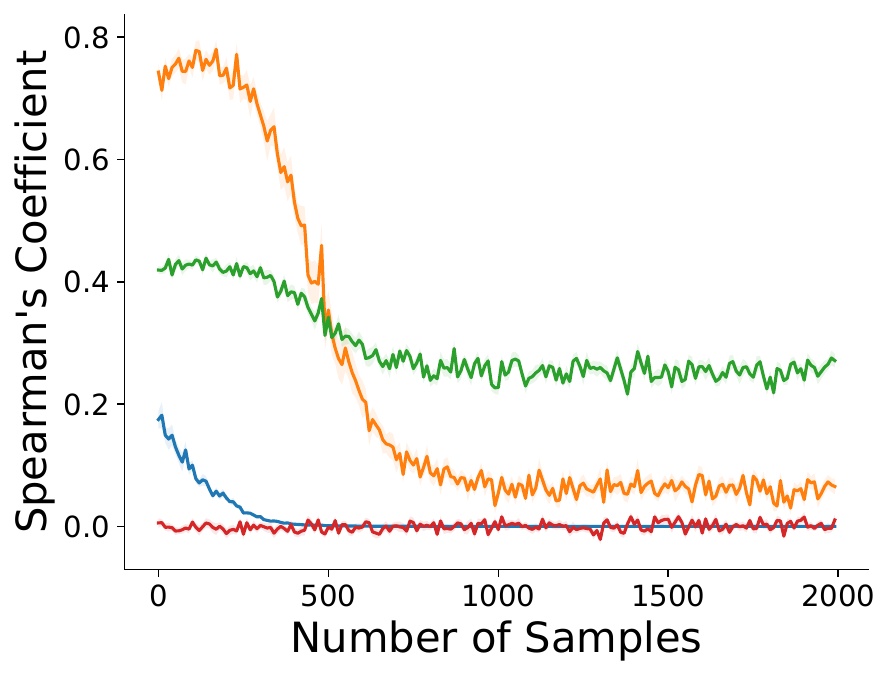}
}

\vspace{0.2cm}
\includegraphics[width=0.65\columnwidth]{figures/Legends/utility.pdf}
\vspace{-0.2cm}
\caption{Spearman Correlation between the true weight/feature utility and approximated weight/feature utilities with different activations}
\label{fig:utility-additional}
\end{figure}


\subsection{UPGD with Second-order approximated Utility on the Toy Problem using Second-order Approximated Utilities}
\label{appendix:so-utility-toy-problem}
We repeat the experiment on the toy problem but with second-order approximated utility. Fig.\ \ref{fig:toy-problem-changing-inputs-all} shows the results on the first variation of the problem, and Fig.\ \ref{fig:toy-problem-changing-outputs-all} shows the results on the second variation of the problem. We note that the second-order approximated utility improves performance over the first-order approximated utility in this problem.

\begin{figure}[ht]
\centering
\subfigure[Weight-wise with Global Utility]{
    \includegraphics[width=0.28\textwidth]{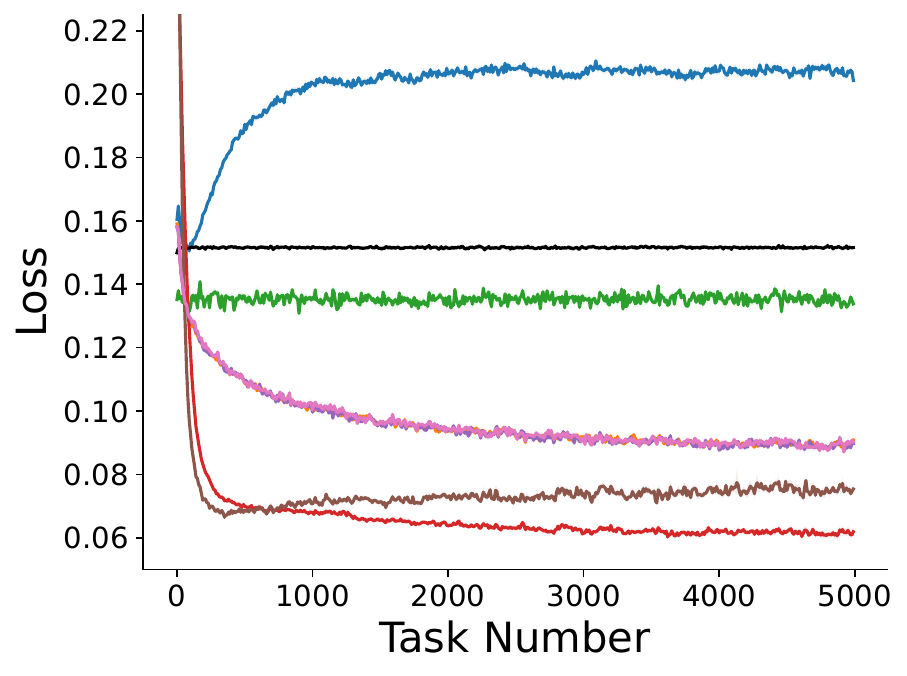}
}
\subfigure[Feature-wise with Global Utility]{
    \includegraphics[width=0.28\textwidth]{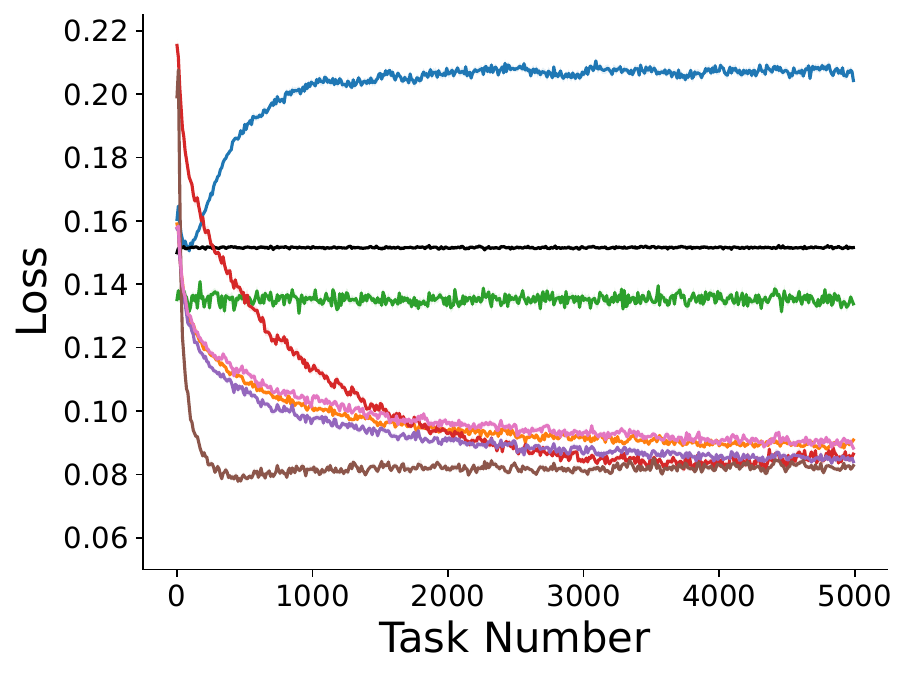}
}

\hfill

\subfigure[Weight-wise with Local Utility]{
    \includegraphics[width=0.28\textwidth]{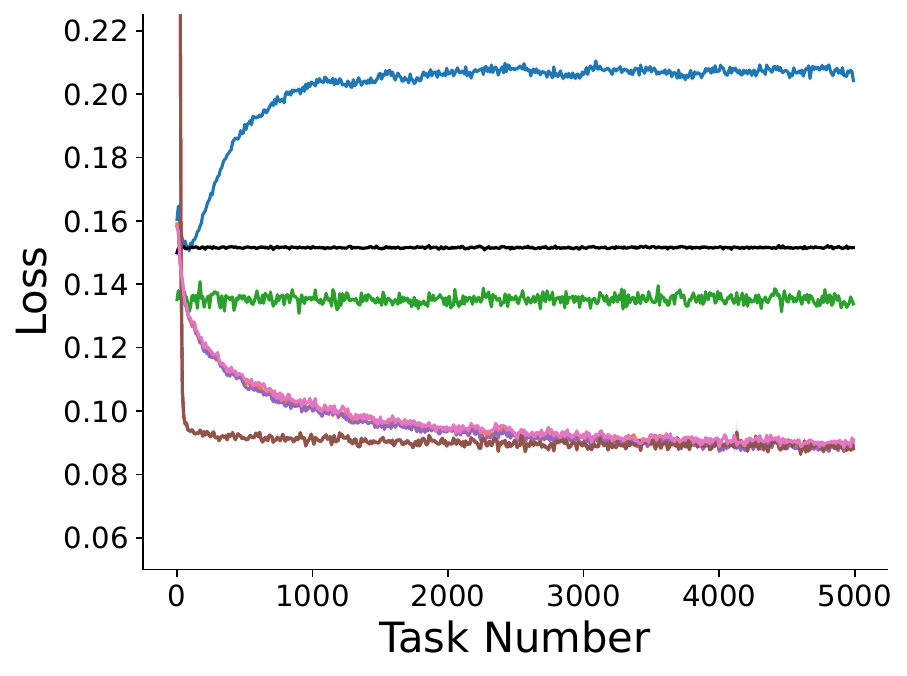}
}
\subfigure[Feature-wise with Local Utility]{
    \includegraphics[width=0.28\textwidth]{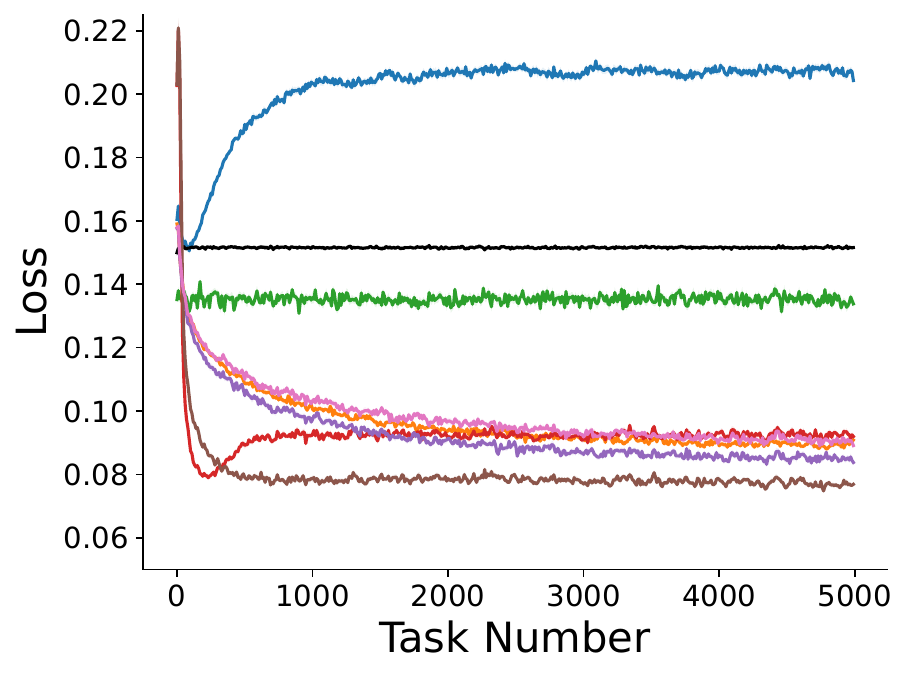}
}

\vspace{0.2cm}
\includegraphics[width=0.9\columnwidth]{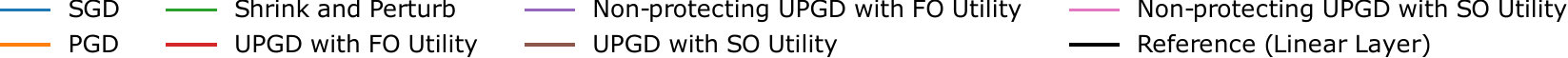}
\vspace{-0.2cm}
\caption{Utility-based Perturbed Gradient Descent with first-order and second-order approximated utilities using uncorrelated noise on the toy problem with changing inputs.}
\label{fig:toy-problem-changing-inputs-all}
\end{figure}

\begin{figure}[ht]
\centering
\subfigure[Weight-wise with Global Utility]{
    \includegraphics[width=0.28\textwidth]{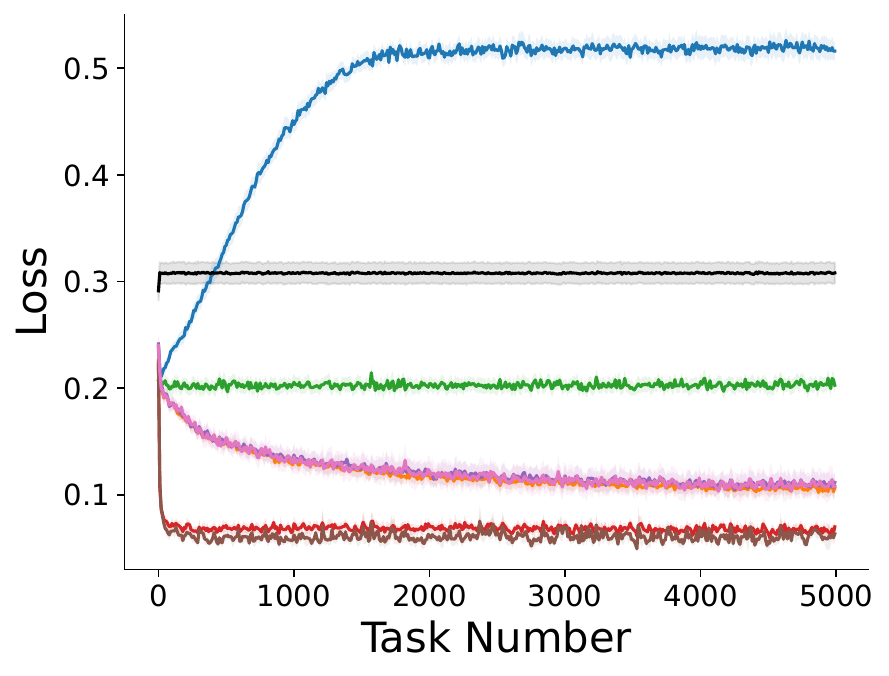}
}
\subfigure[Feature-wise with Global Utility]{
    \includegraphics[width=0.28\textwidth]{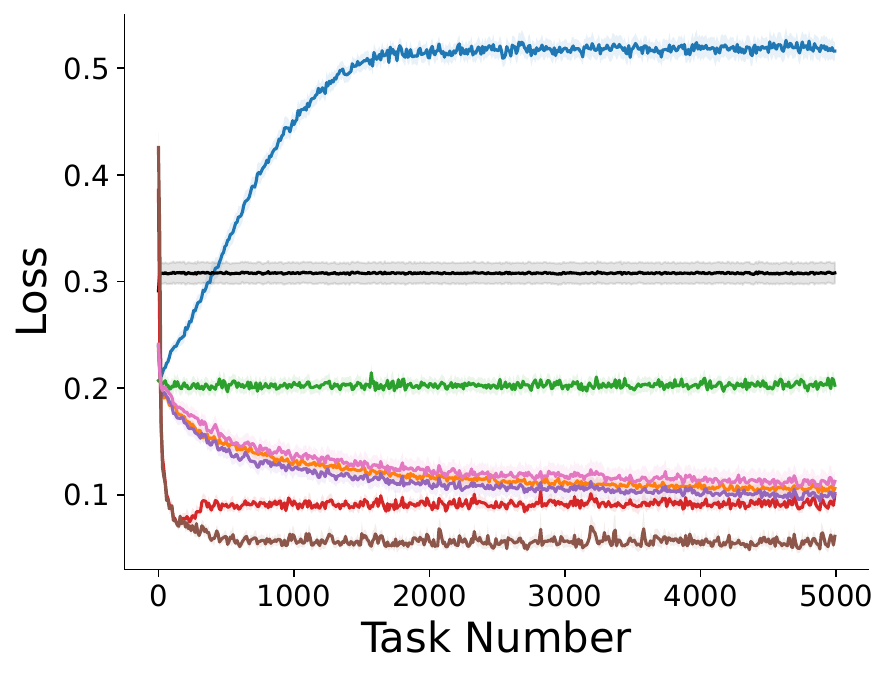}
}

\hfill

\subfigure[Weight-wise with Local Utility]{
    \includegraphics[width=0.28\textwidth]{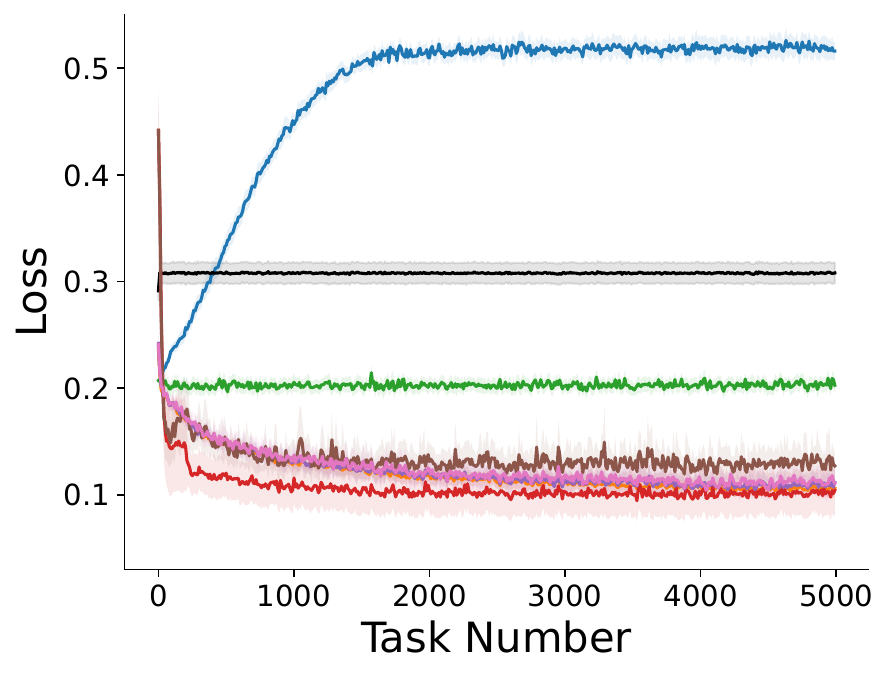}
}
\subfigure[Feature-wise with Local Utility]{
    \includegraphics[width=0.28\textwidth]{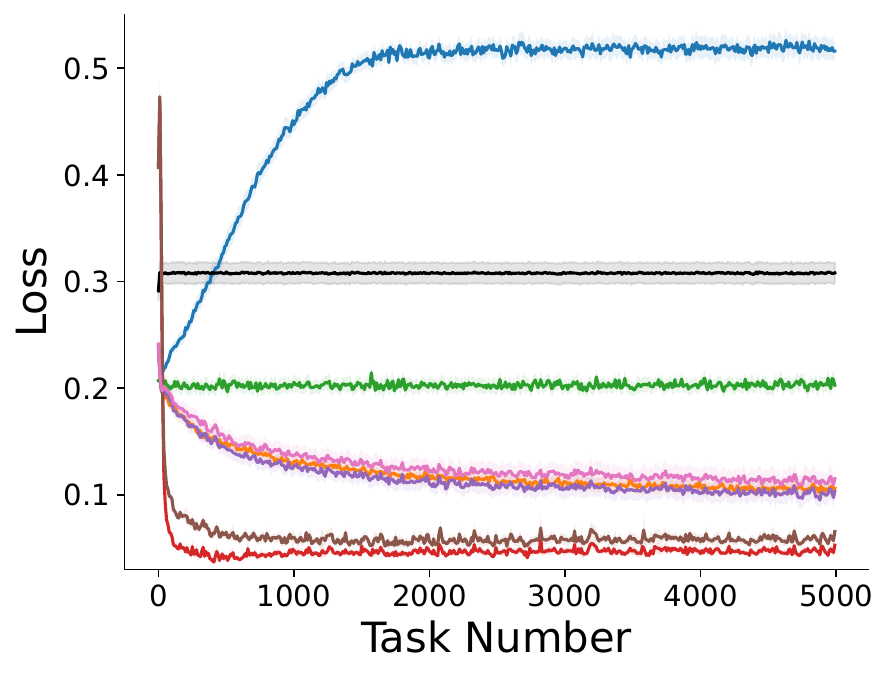}
}

\vspace{0.2cm}
\includegraphics[width=0.9\columnwidth]{figures/Legends/extended-utility.pdf}
\vspace{-0.2cm}
\caption{Utility-based Perturbed Gradient Descent with first-order and second-order approximated utilities using uncorrelated noise on the toy problem with changing outputs.}
\label{fig:toy-problem-changing-outputs-all}
\end{figure}

\subsection{UPGD on the Input-permuted MNIST}
\label{appendix:additional-input-permuted-mnist}
We repeat the experiment on the input-permuted MNIST but with feature-wise approximated utility. In addition, we show the results using the local approximated utility for both weight-wise and feature-wise UPGD in Fig.\ \ref{fig:input-permuted-all}.

\begin{figure}[ht]
\centering
\subfigure[Weight-wise with Global Utility]{
    \includegraphics[width=0.28\textwidth]{figures/InputPermutedMNIST/weight_global.pdf}
}
\subfigure[Feature-wise with Global Utility]{
    \includegraphics[width=0.28\textwidth]{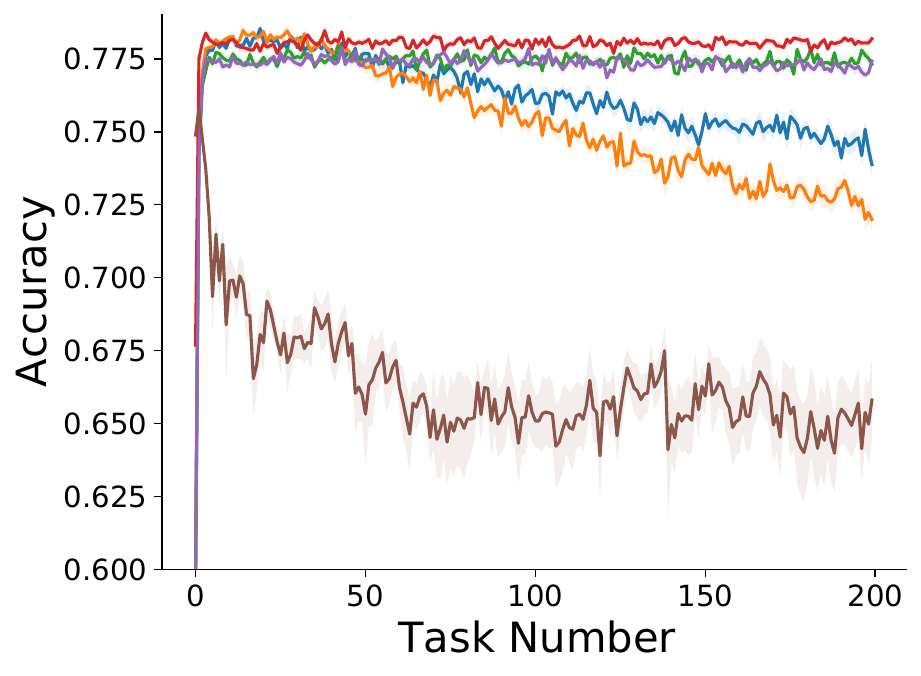}
}

\hfill

\subfigure[Weight-wise with Local Utility]{
    \includegraphics[width=0.28\textwidth]{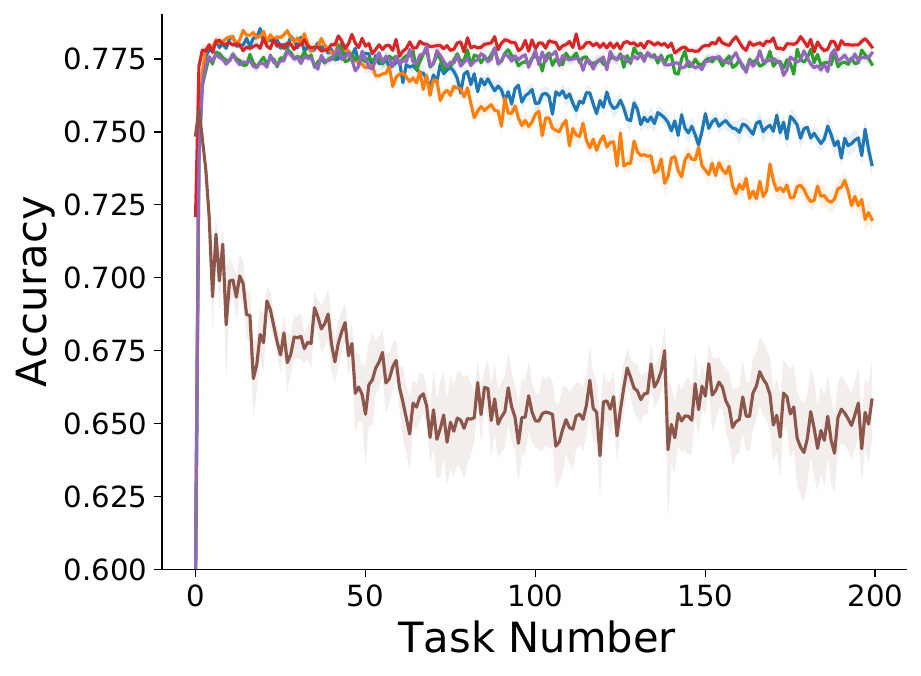}
}
\subfigure[Feature-wise with Local Utility]{
    \includegraphics[width=0.28\textwidth]{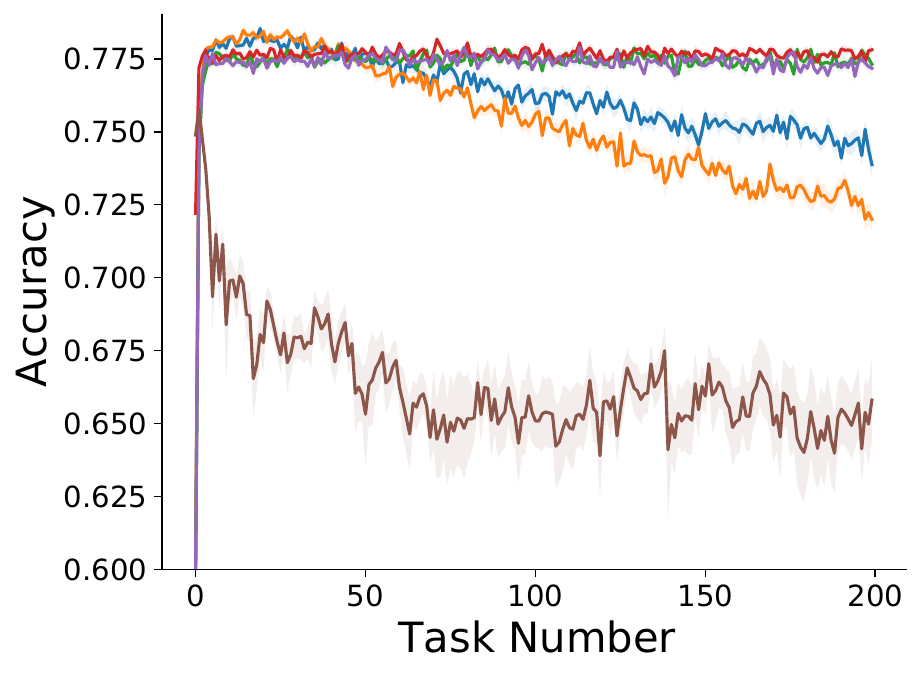}
}

\vspace{0.2cm}
\includegraphics[width=0.5\columnwidth]{figures/Legends/test_legend.pdf}
\vspace{-0.2cm}
\caption{Performance of UPGD, SGD, and Adam with weight decay, PGD and Shrink \& Perturb on Input-Permuted MNIST. A first-order utility is used with the two UPGDs.}
\label{fig:input-permuted-all}
\end{figure}

\subsection{UPGD on Output-permuted EMNIST}
\label{appendix:additional-output-permuted-emnist}
We repeat the experiment on the input-permuted MNIST but with feature-wise approximated utility. In addition, we show the results using the local approximated utility for both weight-wise and feature-wise UPGD in Fig.\ \ref{fig:output-permuted-emnist-all}.

\begin{figure}[ht]
\centering
\subfigure[Weight-wise with Global Utility]{
    \includegraphics[width=0.28\textwidth]{figures/LabelPermutedEMNIST/weight_global.pdf}
}
\subfigure[Feature-wise with Global Utility]{
    \includegraphics[width=0.28\textwidth]{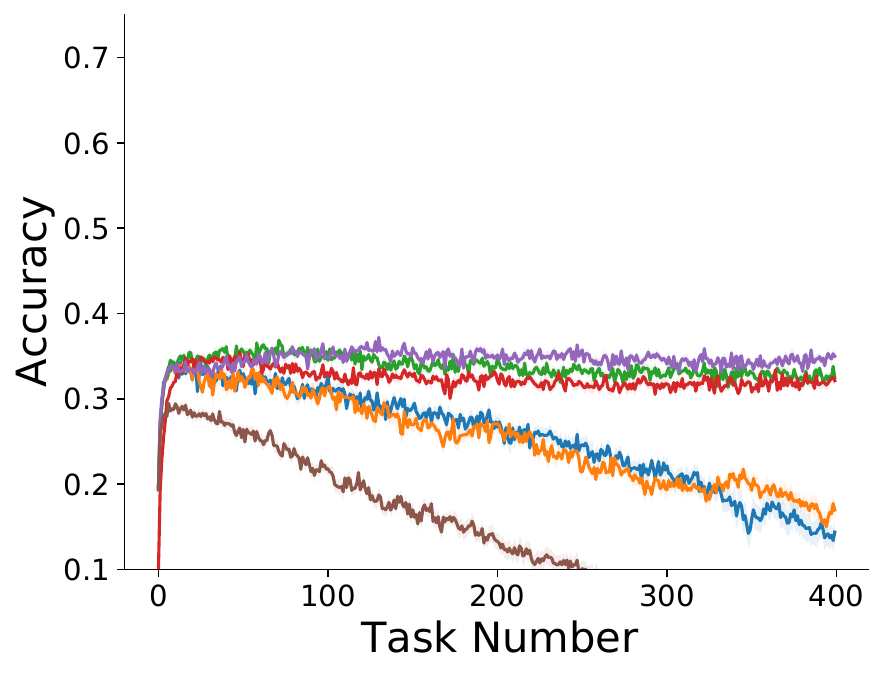}
}

\hfill

\subfigure[Weight-wise with Local Utility]{
    \includegraphics[width=0.28\textwidth]{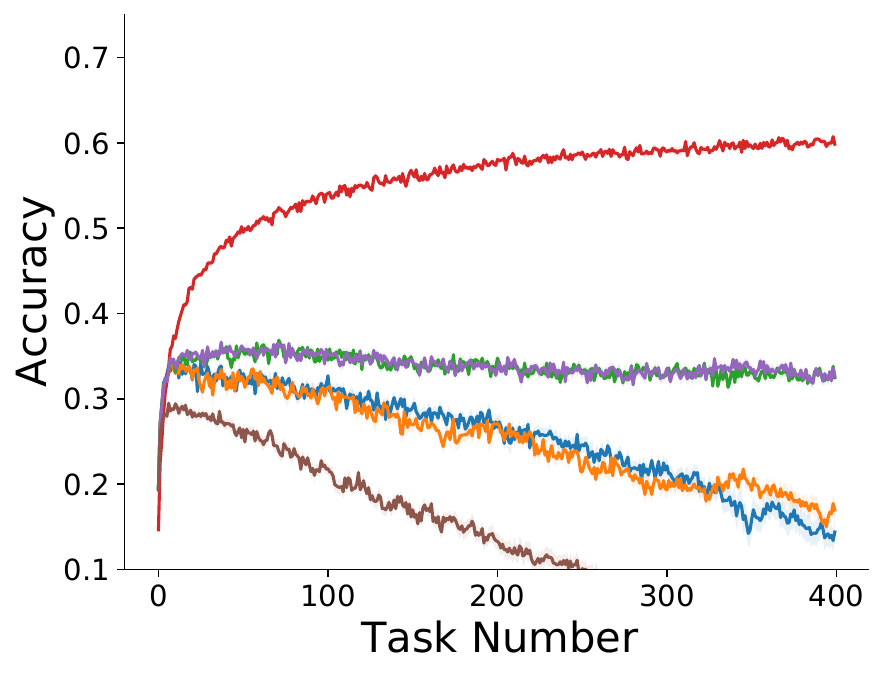}
}
\subfigure[Feature-wise with Local Utility]{
    \includegraphics[width=0.28\textwidth]{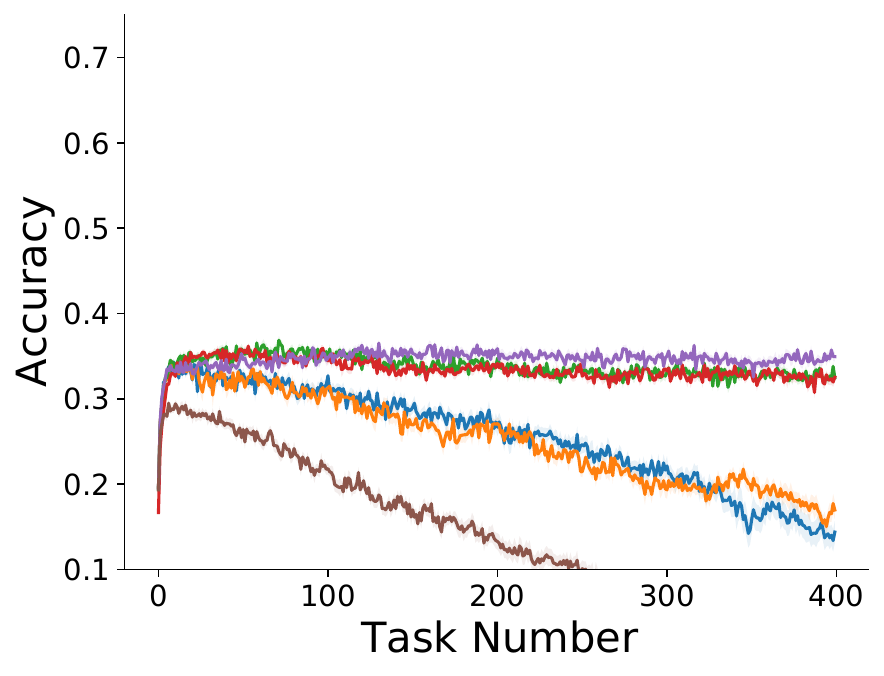}
}

\vspace{0.2cm}
\includegraphics[width=0.5\columnwidth]{figures/Legends/test_legend.pdf}
\vspace{-0.2cm}
\caption{Performance of UPGD, SGD, and Adam with weight decay, PGD and Shrink \& Perturb on Output-Permuted EMNIST. A first-order utility is used with the two UPGDs.}
\label{fig:output-permuted-emnist-all}
\end{figure}

\subsection{UPGD on Output-permuted CIFAR-10}
\label{appendix:additional-output-permuted-cifar10}
We repeat the experiment on the Output-permuted CIFAR-10 but with the local approximated utility. Fig.\ \ref{fig:output-permuted-cifar10-all} shows the results using local and global approximated utilities for UPGD and Non-protecting UPGD. 

\begin{figure}[ht]
\centering
\subfigure[Weight-wise with Global Utility]{
    \includegraphics[width=0.28\textwidth]{figures/LabelPermutedCIFAR10/weight_global.pdf}
}
\subfigure[Weight-wise with Local Utility]{
    \includegraphics[width=0.28\textwidth]{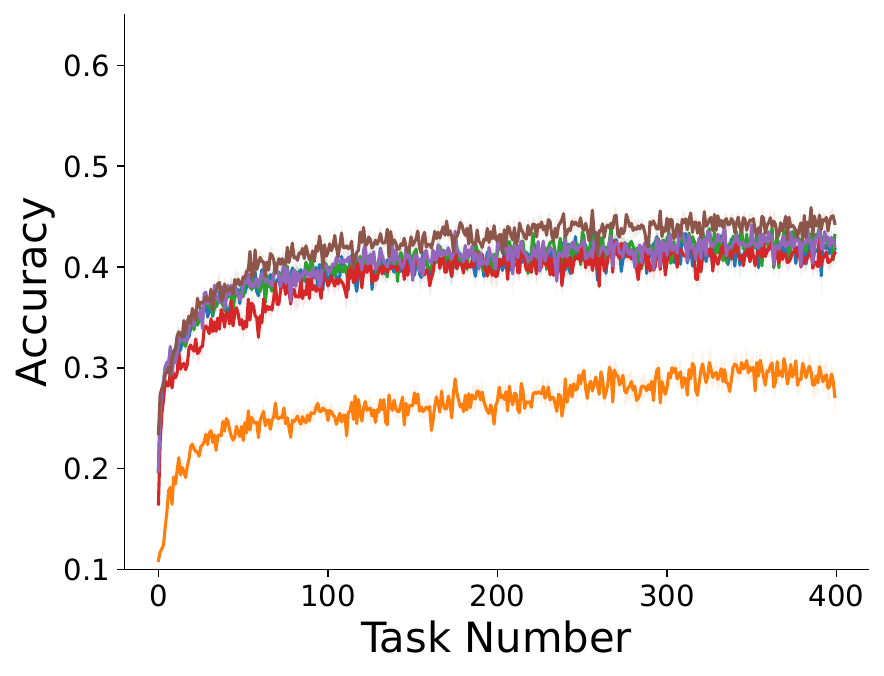}
}

\vspace{0.2cm}
\includegraphics[width=0.5\columnwidth]{figures/Legends/test_legend.pdf}
\vspace{-0.2cm}
\caption{Performance of UPGD, SGD, and Adam with weight decay, PGD and Shrink \& Perturb on Output-Permuted CIFAR-10. A first-order utility is used with the two UPGDs.}
\label{fig:output-permuted-cifar10-all}
\end{figure}


\section{Ablation Study on the Effect of Weight Decay and Weight Perturbation}
\label{appendix:ablation-study}
We conduct a short ablation study on the effect of weight decay and weight perturbation on performance. Fig.\ \ref{fig:ablation-study} shows that compared algorithms with weight decay (first row) and without weight decay (second row). Note that we removed Shrink \& Perturb from the second row and used PGD since Shrink \& Perturb is PGD with weight decay. We notice the significant effect of weight decay on most algorithms. However, UPGD can learn quite well without weight decay. We also notice that only using weight decay without perturbation (e.g., SGDW and AdamW) does not help maintain plasticity in all tasks. We conclude that both weight decay and weight perturbation are important in continual online learning.

\begin{figure}[ht]
\centering
\subfigure[MNIST (weight decay)]{
    \includegraphics[width=0.28\textwidth]{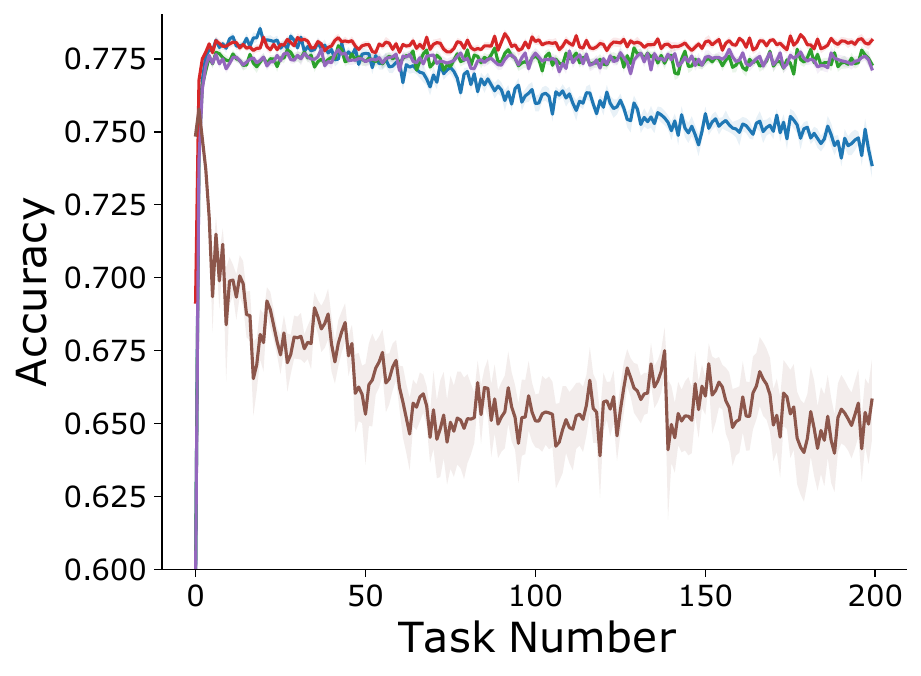}
}
\subfigure[EMNIST (weight decay)]{
    \includegraphics[width=0.28\textwidth]{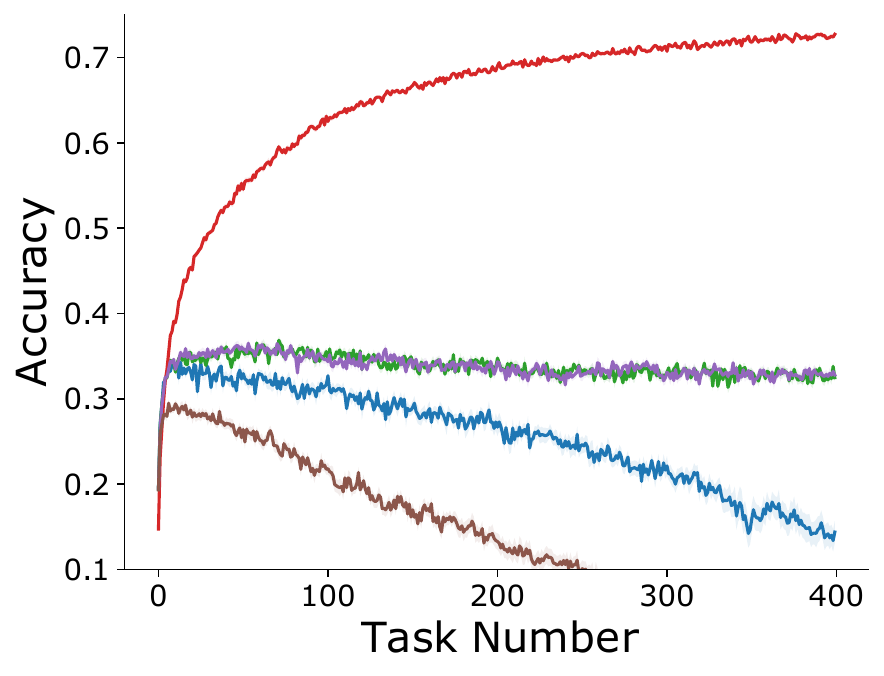}
}
\subfigure[CIFAR-10 (weight decay)]{
    \includegraphics[width=0.28\textwidth]{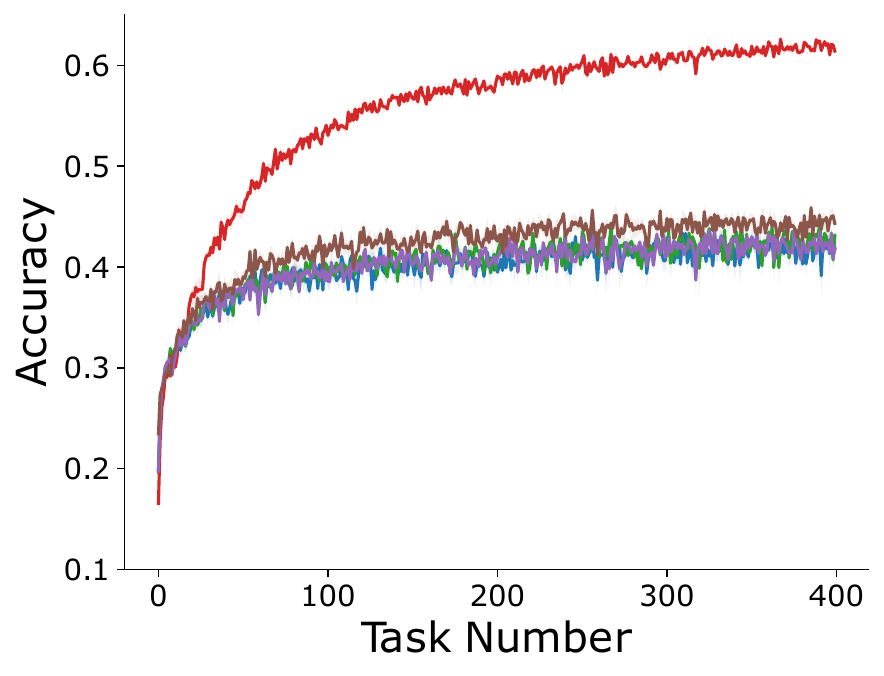}
}

\hfill

\subfigure[MNIST (no decay)]{
    \includegraphics[width=0.28\textwidth]{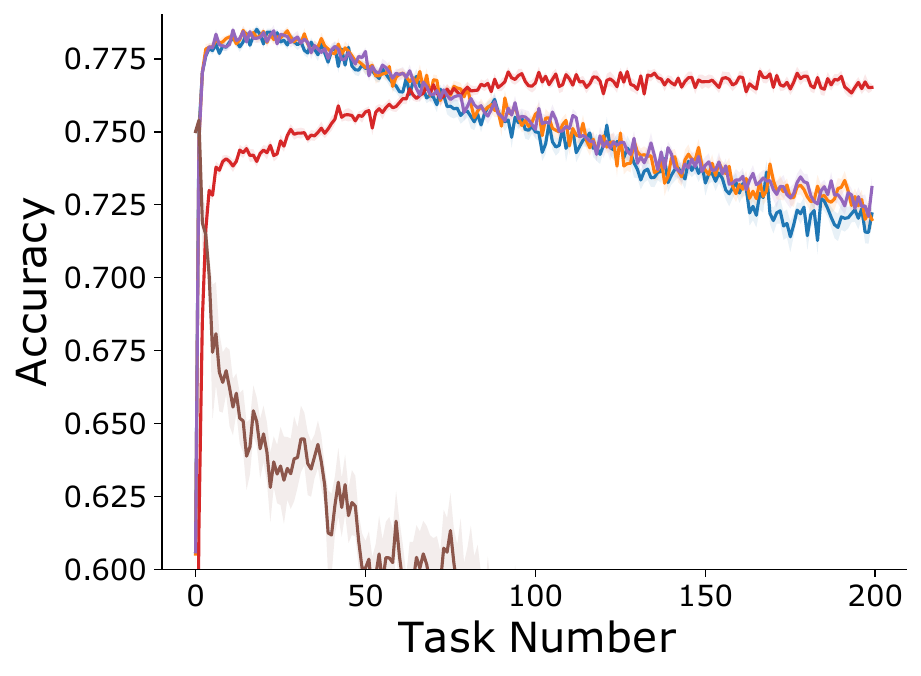}
}
\subfigure[EMNIST (no decay)]{
    \includegraphics[width=0.28\textwidth]{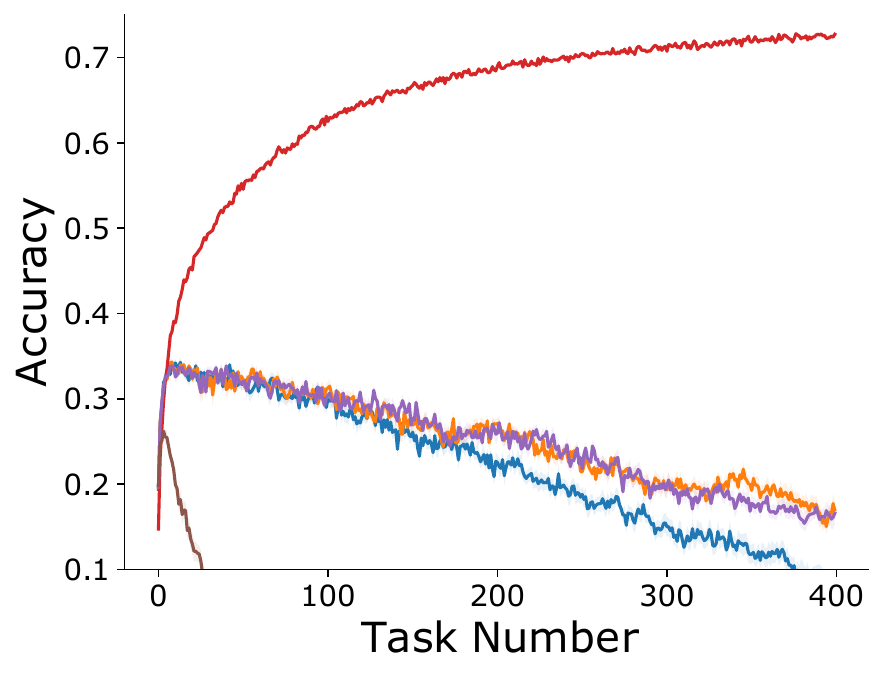}
}
\subfigure[CIFAR-10 (no decay)]{
    \includegraphics[width=0.28\textwidth]{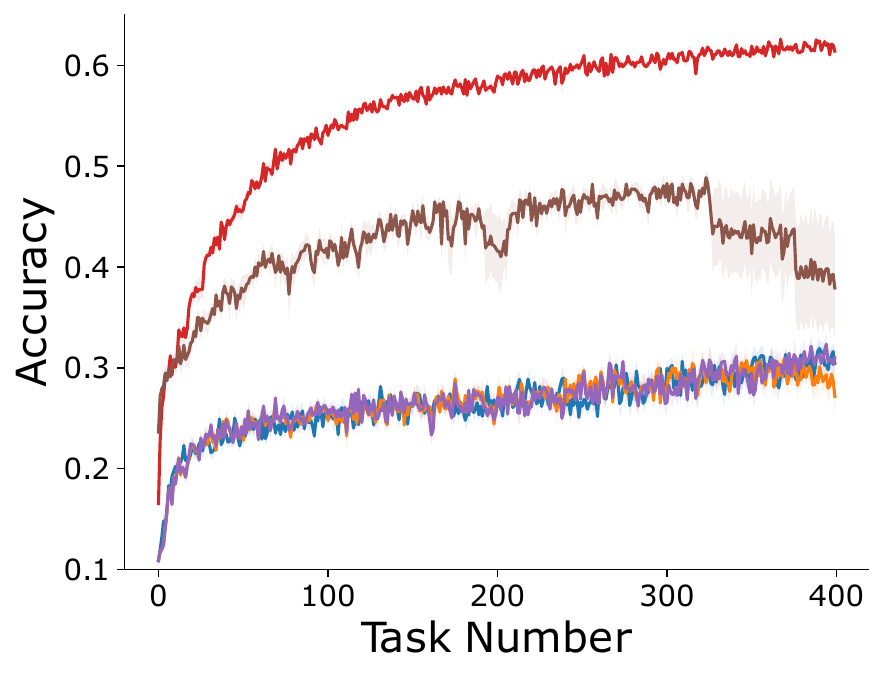}
}
\vspace{0.2cm}
\includegraphics[width=0.5\columnwidth]{figures/Legends/test_legend.pdf}
\vspace{-0.2cm}
\caption{Effect of weight perturbation against weight decay with weight-wise UPGD. Performance of UPGD, SGDW, AdamW, PGD, and Shrink \& Perturb on Input-Permuted MNIST, Output-Permuted EMNIST, and Output-Permuted CIFAR10. A global weight-wise first-order utility is used with the two UPGDs.}
\label{fig:ablation-study}
\end{figure}



\end{document}